%% file: main.tex
\documentclass[10pt,journal,compsoc]{IEEEtran}
\pdfoutput=1

\ifCLASSOPTIONcompsoc
\usepackage[nocompress]{cite}
\else
  \usepackage{cite}
\fi
\usepackage{algorithm}
\usepackage{algpseudocode}
\usepackage{epsfig}
\usepackage{graphicx}
\usepackage{amsfonts}
\usepackage{comment}
\usepackage{caption}
\usepackage{color}
\usepackage{lipsum}
\usepackage{amssymb}
\usepackage{amsmath}
\usepackage{amsthm}
\usepackage{multirow}
\usepackage{multicol}
\usepackage[table]{xcolor}
\usepackage{enumitem}
\usepackage{etoolbox}

 \usepackage[top=0.6in,bottom=0.6in,left=0.56in,right=0.6in]{geometry}
\newtheorem{lemma}{Lemma}
\AfterEndEnvironment{lemma}{\noindent\ignorespaces}

\newtheorem{definition}{Definition}[section]
\AfterEndEnvironment{definition}{\noindent\ignorespaces}
\newtheorem{example}{Example}[section]
\AfterEndEnvironment{example}{\noindent\ignorespaces}

\ifCLASSINFOpdf
\else
\fi
\begin{document}
\title{QoS aware Automatic Web Service Composition with Multiple objectives}

\author{Soumi~Chattopadhyay,~\IEEEmembership{Student Member,~IEEE,} and
        Ansuman~Banerjee,~\IEEEmembership{Member,~IEEE}

\IEEEcompsocitemizethanks{\IEEEcompsocthanksitem Authors with Indian Statistical Institute, Email: $ansuman@isical.ac.in$}}

\makeatletter
\long\def\@IEEEtitleabstractindextextbox#1{\parbox{0.922\textwidth}{#1}}
\makeatother

\IEEEtitleabstractindextext{
\begin{abstract}
With increasing number of web services, providing an end-to-end Quality of Service 
(QoS) guarantee in responding to user queries is becoming an important concern. Multiple QoS parameters (e.g., response time, latency, throughput, 
reliability, availability, success rate) are associated with a service, thereby, service composition with a 
large number of candidate services is a challenging multi-objective optimization problem. In this paper, we study 
the multi-constrained multi-objective QoS aware web service composition problem and propose three different approaches to solve the same, one optimal, based on Pareto front construction and two other based on heuristically traversing the solution space.
We compare the performance of the heuristics against the optimal, and show the effectiveness 
of our proposals over other classical approaches for the same problem setting, with experiments on WSC-2009 and ICEBE-2005 datasets.
\end{abstract}

}

\maketitle

\IEEEdisplaynontitleabstractindextext
\IEEEpeerreviewmaketitle
\ifCLASSOPTIONcaptionsoff
  \newpage
\fi

\input{intro}

\input{related}

\input{Preliminaries}

\input{overview}
\input{Detailed_Methodology}

\input{result}

\input{conclusion}

\scriptsize
\bibliographystyle{IEEEtran}
\bibliography{ref}

 \vspace{-0.6in}
\scriptsize
%

\newpage
\input{appendix}

\end{document}

%% file: intro.tex
\IEEEraisesectionheading{\section{Introduction}\label{sec:intro}}
\noindent
In recent times, web services have become ubiquitous with the proliferation of Internet usage. A web service is a software component, that takes a set of inputs, performs a specific task and produces a set of outputs. A set of non-functional quality of service (QoS) parameters (e.g., response time, throughput, reliability etc.) are associated with a web service. These QoS parameters 
determine the performance of a service. Sometimes, a single web service falls short to respond to a user query. Therefore, service composition~\cite{el2010tqos,oh2008effective} is required. During service composition, multiple services are combined in a specific order based on their input-output dependencies to produce a desired set of outputs. While providing a solution in response to a query, it is also necessary to ensure fulfillment of end-to-end QoS requirements~\cite{Zeng:2003:QDW:775152.775211}, which is the main challenge in QoS aware service composition \cite{7226855,bartalos2012automatic,peer2005web}. 

A large body of literature in service composition deals with optimization of a single QoS parameter \cite{rodriguez2015hybrid,6009375}, especially, response time or throughput. However, a service may have multiple QoS parameters; therefore, the service composition problem turns out to be a multi-objective optimization problem. Though optimality of the end solution is the primary concern in multi-objective service composition, computing the optimal solution is time consuming. This has led to another popular research theme around multi-constrained service composition \cite{Alrifai:2012:HAE:2180861.2180864,1357986}, where a constraint is specified on each QoS parameter and the objective is to satisfy all the QoS constraints in the best possible way. 


Two different models have been considered in service composition literature, namely, workflow based model (WM) \cite{guidara2015heuristic,qi2010combining} and input-output dependency based model (IOM) \cite{rodriguez2015hybrid}. The salient features of both the models are discussed in Table \ref{tab:features}. 
Most of the research proposals on multiple-QoS aware service composition have considered WM \cite{ba2016exact,zhang2013selecting}. In general, the methods proposed in WM cannot solve the problem in IOM, since in IOM, in addition to the QoS values of a service, the input-output dependencies between the service also need to be considered. Approaches \cite{feng2007model,hu2005quality}, that consider IOM typically transform the multiple objectives into a single objective and generate the optimal solution instead of the Pareto optimal solutions \cite{6985731}. A weighted sum is used to convert multiple objectives into a single objective. However, finding the weights is a challenging task. 

\input{features}

In this paper, we study the multi-objective QoS-aware web service composition problem in IOM. To the best of our knowledge, there is not much work in IOM considering multiple QoS aware service composition based on Pareto front construction. However, considering the parameters individually instead of a weighted sum combination, has a major significance, since it can deal with the users having various QoS preferences. Additionally, we have considered multiple local and global constraints on different QoS parameters. In this paper, our major contributions are as follows:

\noindent
 ~\textbullet~ We first propose an optimal algorithm, that constructs the Pareto optimal solution frontier satisfying all QoS constraints for the multi objective problem in IOM. We theoretically prove the soundness and completeness of our algorithm. 

\noindent
~\textbullet~ Additionally, we propose two heuristics. The first one employs a beam search strategy \cite{russell1995modern}, while the other is based on non deterministic sorting genetic algorithm (NSGA) \cite{deb2002fast}. 
 
\noindent
~\textbullet~ To demonstrate the time-quality trade-off, we perform extensive experiments on the benchmarks ICEBE-2005 \cite{icebe2005} and WSC-2009 \cite{bansal2009wsc}. Additionally, we compare our proposed methods with \cite{yan2015anytime}, which proposes the composition problem in IOM using a single objective weight transformation.

\noindent
This paper is organized as follows. In Section \ref{sec:related}, we compare and contrast our model and proposed approaches with the existing literature. Section \ref{sec:preliminaries} presents some background, the next outlines our problem. Section \ref{sec:method}-\ref{sec:ga} present our proposal, Section \ref{sec:result} presents results. Section~\ref{sec:conclusion} concludes the work.

%% file: features.tex
\begin{table*}[!ht]
\scriptsize
\caption{Salient features of WM and IOM}
\centering
\begin{tabular}{l|l|l}
 \hline
 \multicolumn{1}{c|}{Features} & \multicolumn{1}{c|}{WM} & \multicolumn{1}{c}{IOM}\\
 \hline
 Query specification & A Workflow: a set of tasks to be performed in a specific order & A set of given query inputs and a set of desired query outputs \\
 \hline
 Query objective & To serve the query by selecting a service for each task so that & To serve the query by identifying a set of services that are directly \\
               & the overall QoS values are optimized, where the service & (by the query inputs) or indirectly (by the outputs of the services that \\
               & repository contains a set of functionally equivalent services & are directly or indirectly activated by the query inputs) activated by the\\
               & for each task                                 & query inputs and can produce the query outputs\\
 \hline
 Search Space & $m^n$, $n=$ the number of tasks and $m=$ the number of functionally  & $2^k$ (in the worst case), $k=$ the total number of services that can be \\
              & equivalent services for each task; Total number of services & activated by the query inputs\\
              & that can participate to serve the query $k= m \times n$ & \\
 \hline
 Complexity & For a single QoS parameter, finding the optimal solution is a & Though each of response time and throughput, when treated as an individual \\ 
            & polynomial time algorithm \cite{abu2015complexity}. For multiple QoS parameters, & parameter can be optimized in polynomial time, some of the other parameters \\
            & finding the Pareto optimal solutions is NP-hard \cite{Alrifai2010} & (e.g., reliability, price, availability) require exponential time procedures, even\\
            & & when a single parameter optimization is considered. Multiple parameters  \\
            &&and their simultaneous optimization turns out to be a hard problem \cite{yan2012anytime} \\
 \hline
 \end{tabular}\label{tab:features}
\end{table*}

%% file: related.tex
\section{Related Work}\label{sec:related}
\noindent
Automatic service composition \cite{7027801,hwang2008dynamic} is a fundamental problem in services computing. A significant body of research has been carried out on QoS-aware service composition considering a single QoS parameter, especially, response time and throughput \cite{chattopadhyay2015scalable,xia2013web,chen2014qos}. Multiple-QoS aware service composition has been discussed in \cite{mostafa2015multi}. We first discuss related work regarding the models followed by the solution approaches for multiple QoS aware web service composition.

\subsection{Problem Models}\label{subsec:problemmodel}
\noindent
The two most popular models considered in literature are the workflow model (WM) and input-output dependency based model (IOM). The salient features of these two models are discussed in Table \ref{tab:features}.
In WM, it is assumed that a task can be accomplished by a single web service. However, in practice, it may not be the case always. Some times more than one service may be required to perform a particular task. Therefore, the input-output dependency based model becomes popular.


It may be noted, existing solution approaches for WM are unable to solve the composition problem in IOM. This is mainly because of the following reasons:
\begin{itemize}
 \item In WM, a workflow is provided as an input, whereas, in IOM, no workflow is provided, rather the aim of IOM is to find out a flow of services to serve the query so that the overall QoS values are optimized.
 \item In general, while selecting a service in WM, only the QoS values of the service need to be taken care of. In contrast to the former case, while selecting a service in IOM, not only the QoS values of the services need to be considered but also its input-output dependencies on the other services need to be taken into account.
 \item In WM, the number of services does not vary across all solutions, while, in IOM, the number of services varies across the solutions to a query.
\end{itemize}
However, methods that can solve the composition problem in IOM can solve the composition problem in WM. Moreover, the search space of WM is a subset of the search space of IOM. We now discuss different solution models.

\subsection{Solution Models and Approaches}\label{subsec:solutionmodel}
\noindent
We classify below the different solution models and discuss the approaches existing in literature.

\noindent
{\bf{Scalarization (SOO)}}: 
To deal with multiple QoS aware service composition, \cite{qi2010combining,wagner2011qos,DBLP:journals/tweb/ChattopadhyayBB17} have resorted to scalarization techniques to convert multiple objectives into a single objective using the weighted sum method. In \cite{yan2015anytime}, the authors have proposed a planning graph based approach and an anytime algorithm that attempts to maximize the utility in IOM. A scalarization technique in WM was proposed in \cite{qi2010combining}. Though scalarization techniques are simple and easy to implement, however, some information may be lost due to the transformation from multiple objectives to a single objective. Moreover, finding the weights of the parameters is difficult. User preferences are required to decide the weights of the parameters, which is not always easy to identify. Even though the preferences of the parameters are obtained, it is not easy to quantify the preferences to find the weights of the parameters, which has a great impact on finding the optimal solution. 

%

\noindent
{\bf{Single-objective multi-constrained optimization (SOMCO)}}: To overcome the shortcomings of scalarization techniques, researchers have looked at another popular approach, namely, single-objective multi-constrained optimization \cite{Alrifai2010,alrifai2009combining,DBLPChattopadhyayB16}. In this approach, one parameter is selected as the primary parameter to be optimized, while for the rest of the parameters, a worst case bound is set (often termed as constraints). For example, in \cite{cao2007service}, the authors analyzed the relation between multi-objective service composition and the Multi-choice, Multi-dimension 0-1 Knapsack Problem (MMKP) in WM and used the weighted sum approach to compute the utility function. The objective of \cite{cao2007service} is to maximize the total utility while satisfying different QoS constraints. In \cite{Alrifai2010} and \cite{alrifai2009combining}, authors proposed a multi constrained QoS aware service composition approach, instead of finding the optimal solutions in WM. In \cite{Alrifai2010}, an Integer Linear Programming (ILP) based approach was proposed, where ILP is used to divide the global constraints into a set of local constraints and then using the local constraints, the service selection is done for each task in WM. In \cite{schuller2012cost,Zeng:2003:QDW:775152.775211}, ILP based methods are used to solve multi-constrained service composition in WM. Dynamic binding is the main concern of \cite{alrifai2009combining}, where authors proposed to generate the skyline services for each task in WM and cluster the services using the K-means algorithm. In \cite{DBLPChattopadhyayB16}, an ILP based multi-constrained service composition was proposed in IOM.   
In this class of methods as well, selecting the primary parameter to optimize (and rest to put constraints on) is a challenging problem and often depends on user preferences. Moreover, determining the constraint values is not an easy task and may this often lead to no solution being generated (i.e., no solution exists to satisfy all the constraints).

\noindent
{\bf{Pareto optimal front construction (POFC)}}: To address the above challenges, another research approach based on constructing the Pareto optimal frontier has been proposed. A Pareto front consists of the set of solutions where each solution is either same or better in at-least one QoS value than rest of the solutions belonging to the Pareto front. This approach does not require identifying the user preferences of the QoS parameters. Therefore, this approach can easily deal with users having different preferences. To the best of our knowledge, most of the work based on Pareto front construction \cite{schuller2012cost} focus on WM. For example, in \cite{6985731}, the authors proposed to generate the Pareto optimal solutions in a parallel setting. In \cite{trummer2014multi}, the authors proposed a fully polynomial time approximation method to solve the problem. A significant amount of work has been done based on evolutionary algorithms \cite{cremene2016comparative,li2010applying}, such as Particle Swarm Optimization \cite{liao2013multi}, Ant Colony Optimization \cite{shanshan2012improved}, Bee Colony Optimization \cite{liu2014parameter}, Genetic Algorithms \cite{zhang2013genetic,wang2008optimal}, NSGA2 \cite{hashmi2013automated,wagner2012multi}. In this paper, we consider the Pareto front construction model on IOM.

Table \ref{tab:relatedModel} summarizes the state-of-the-art methods considering different models that have been discussed above.

\begin{table}[!ht]
\scriptsize
\caption{State of the Art regarding Models}
\centering
\begin{tabular}{c|c}
 \hline
  Models & Methods \\
 \hline
 WM-SOO & \cite{qi2010combining}\\\hline
 WM-SOMCO & \cite{Alrifai2010,alrifai2009combining,cao2007service,Zeng:2003:QDW:775152.775211,zhang2013genetic,wang2008optimal,liu2014parameter,liao2013multi,shanshan2012improved,schuller2012cost} \\\hline
 WM-POFC & \cite{6985731,mostafa2015multi,cremene2016comparative,hashmi2013automated,wagner2012multi,trummer2014multi,li2010applying} \\\hline
 IOM.SOO & \cite{yan2015anytime,wagner2011qos,DBLP:journals/tweb/ChattopadhyayBB17} \\\hline
 IOM-SOMCO & \cite{DBLPChattopadhyayB16} \\\hline
 IOM-POFC & - \\
 \hline
\end{tabular}\label{tab:relatedModel}
\end{table}




\subsection{Novelty of our work and contributions}
\noindent
In contrast to the above, we consider this problem in IOM-POFC setting. In addition, we have considered local and global constraints on QoS parameters. In Section \ref{sec:preliminaries}, we formally describe our model.

The search space of the composition problem addressed in this paper is exponential as discussed earlier. Therefore, we first try to reduce the search space of our algorithm using clustering as demonstrated in \cite{7933195,wagner2011qos}. On the reduced search space, we propose an optimal algorithm using a graph based method. In literature, the graph based methods \cite{DBLP:journals/tweb/ChattopadhyayBB17,xia2013web,chen2014qos} are mainly applied either to solve the service composition problem for single parameter optimization or to solve the multiple QoS aware problem using scalarization. In this paper, we apply the graph based approach to construct a Pareto optimal solution frontier. The optimal algorithm is an exponential time procedure and often does not scale for large scale composition. Therefore, we further propose two heuristic algorithms. 

Our first heuristic algorithm is based on beam search technique. Beam search technique is applied in \cite{yan2015anytime} to solve the multiple QoS aware problem using scalarization. Here, we use the beam search technique to find Pareto optimal solutions. Since our algorithm is a heuristic approach, it does not generate the optimal solutions. However, we have shown that the solution quality monotonically improves with increase in the size of the beam width. 

Our second heuristic algorithm is based on NSGA. Though multiple evolutionary algorithms exist in literature \cite{zhang2013genetic,wang2008optimal,hashmi2013automated,wagner2012multi} to solve multiple QoS aware optimization, however, all these methods, to the best of our knowledge, solve the problem in WM. We use it to find the solutions for IOM. Moreover, in each step of the algorithm based on NSGA, we ensure that the solutions generated by the algorithm is a functionally valid solution. 

%% file: Preliminaries.tex
\section{Background and Problem Formulation}\label{sec:preliminaries}
\noindent
In this section, we discuss some background concepts for our work.
We begin with a classification of a QoS parameter.

\begin{definition}{\em [\bf Positive / Negative QoS parameter:]}
 A QoS parameter is called a positive (negative) QoS parameter, if a higher (lower) value of the parameter implies better performance.
 \hfill$\blacksquare$
\end{definition}

\noindent
Reliability, availability, throughput are examples of positive QoS parameters, while response time, 
latency are examples of negative QoS parameters. 

Consider two services ${\cal{W}}_i$ and ${\cal{W}}_j$ being compared with respect to a QoS 
parameter ${\cal{P}}_k$. We have the following cases:

\begin{itemize}
 \item ${\cal{W}}_i$ is {\em{better}} than ${\cal{W}}_j$ with respect to ${\cal{P}}_k$ implies,
 \begin{itemize}
  \item If ${\cal{P}}_k$ is a positive QoS, ${\cal{P}}^{(i)}_k > {\cal{P}}^{(j)}_k$, where 
	${\cal{P}}^{(i)}_k$ and ${\cal{P}}^{(j)}_k$ are the respective values of ${\cal{P}}_k$ for ${\cal{W}}_i$ and ${\cal{W}}_j$.
  \item If ${\cal{P}}_k$ is a negative QoS parameter, ${\cal{P}}^{(i)}_k < {\cal{P}}^{(j)}_k$.
 \end{itemize} 
 \item ${\cal{W}}_i$ is {\em{as good as}} ${\cal{W}}_j$ with respect to ${\cal{P}}_k$ implies,
	${\cal{P}}^{(i)}_k = {\cal{P}}^{(j)}_k$, irrespective of whether ${\cal{P}}_k$ is positive or negative.
 \item ${\cal{W}}_i$ is {\em{at least as good as}} ${\cal{W}}_j$ with respect to ${\cal{P}}_k$ 
	implies, either ${\cal{W}}_i$ is {\em{better}} than ${\cal{W}}_j$ or ${\cal{W}}_i$ is {\em{as good as}} 
	${\cal{W}}_j$ with respect to ${\cal{P}}_k$.
\end{itemize}

\noindent
The QoS parameters are further classified into four categories based on the aggregate functions used for composition: maximum, minimum, addition, multiplication. 

A query is specified in terms of a set of input-output parameters. We now present the concept of eventual activation of a web service for a given query. A web service is {\em activated}, when the set of inputs of the service is available in the system.
As an example, consider ${\cal{W}}_1$ in Table \ref{tab:exampleServices}, ${\cal{W}}_1$ is activated when its input $i_1$ is available. A service ${\cal{W}}_i$ is {\em eventually activated} by a set of input parameters $I$, if ${\cal{W}}_i$ is either directly activated by $I$ itself or indirectly activated by the outputs of the set of services that are eventually activated by $I$, as shown in Example~\ref{exm:overview}. In the next subsection, we formally discuss the model considered in this paper and our objective.

\subsection{Problem Formulation}\label{subsec:problemformulation}
\noindent
The service composition problem considered in this paper can be formally described as below:

\begin{itemize}
 \item A set of web services $W = \{{\cal{W}}_1, {\cal{W}}_2, \ldots, {\cal{W}}_n\}$
 \item For each service ${\cal{W}}_i \in W$, a set of inputs ${\cal{W}}_i^{ip}$ and a set of outputs ${\cal{W}}_i^{op}$
 \item A set of QoS parameters ${\cal{P}} = \{{\cal{P}}_1, {\cal{P}}_2, \ldots, {\cal{P}}_m\}$
 \item For each service ${\cal{W}}_i \in W$, a tuple of QoS values 
	${\cal{P}}^{(i)} = ({\cal{P}}^{(i)}_1, {\cal{P}}^{(i)}_2, \ldots, {\cal{P}}^{(i)}_m)$ 
 \item A set of aggregation functions ${\cal{F}} = \{f_1, f_2, \ldots, f_m\}$, where $f_i$ is defined for a QoS parameter
	${\cal{P}}_i \in {\cal{P}}$  
 \item A query ${\cal{Q}}$, specified by a set of inputs ${\cal{Q}}^{ip}$ and a set of requested outputs ${\cal{Q}}^{op}$
 \item Optionally, a set of local QoS constraints ${\cal{LC}} = \{{\cal{LC}}_1, {\cal{LC}}_2, \ldots, {\cal{LC}}_{k_l}\}$ and 
	a set of global QoS constraints ${\cal{GC}} = \{{\cal{GC}}_1, {\cal{GC}}_2, \ldots, {\cal{GC}}_{k_g}\}$
\end{itemize}

\noindent
A constraint denotes a bound on the worst case value of a QoS parameter.
While the {\em{local constraints}} are applicable on a single service (${\cal{LC}}_1$
in Example \ref{exm:overview}), the {\em{global constraints}} are applicable on a composition 
solution, (${\cal{GC}}_1$ in Example \ref{exm:overview}). 

The objective of multi-objective QoS constrained service composition is to serve ${\cal{Q}}$ in a way such that the QoS values are optimized, while ensuring functional dependencies are preserved, and all local and global QoS constraints are satisfied. Since multiple (and often disparate) QoS parameters are involved, this calls for a classical multi-objective optimization, and we address this challenge in this work.  In this paper, we propose an optimal solution construction methodology.  Often, a single solution may not be the best with respect to all the QoS parameters. Therefore, instead of producing a single solution, our method generates a set of Pareto optimal solutions as described in the following sections.

%% file: overview.tex
\subsection{Running Example}
\noindent
We now present an illustrative example for our problem. 

\begin{example}\label{exm:overview}

Table \ref{tab:exampleServices} shows a brief description of the services in a service repository,
their inputs, outputs and values of response time (in ms), throughput (number of service invocations per minute) and reliability (in percentage) in the form of a tuple
(RT, T, R).

\begin{table}[!ht]
\tiny
\caption{{Description of example services}}
\centering
\begin{tabular}{c|c|c|c}
 Services & Inputs & Outputs & (RT, T, R)\\
 \hline\hline
 ${\cal{W}}_1$ & $\{i_1\}$ & $\{io_4, io_5\}$ & ${\cal{P}}^{(1)}:$ (500, 7, 93\%)\\
 \hline
 ${\cal{W}}_2$ & $\{i_1\}$ & $\{io_4, io_5\}$ & ${\cal{P}}^{(2)}:$ (600, 13, 69\%)\\
 \hline
 ${\cal{W}}_3$ & $\{i_1\}$ & $\{io_4, io_5\}$ & ${\cal{P}}^{(3)}:$ (350, 4, 97\%)\\
 \hline
 ${\cal{W}}_4$ & $\{i_1\}$ & $\{io_4, io_5\}$ & ${\cal{P}}^{(4)}:$ (475, 3, 85\%)\\
 \hline
 ${\cal{W}}_5$ & $\{i_2\}$ & $\{io_6\}$ & ${\cal{P}}^{(5)}:$ (1300, 15, 81\%)\\
 \hline
 ${\cal{W}}_6$ & $\{i_2\}$ & $\{io_6\}$ & ${\cal{P}}^{(6)}:$ (700, 19, 90\%)\\
 \hline
 ${\cal{W}}_7$ & $\{i_1, i_2\}$ & $\{o_{14}\}$ & ${\cal{P}}^{(7)}:$ (1100, 9, 80\%)\\
 \hline
 ${\cal{W}}_8$ & $\{i_3\}$ & $\{io_4\}$ & ${\cal{P}}^{(8)}:$ (1100, 6, 73\%)\\
 \hline
 ${\cal{W}}_9$ & $\{i_3\}$ & $\{io_4\}$ & ${\cal{P}}^{(9)}:$ (300, 13, 79\%)\\
 \hline
 ${\cal{W}}_{10}$ & $\{i_3\}$ & $\{io_4\}$ & ${\cal{P}}^{(10)}:$ (800, 9, 78\%)\\
 \hline
 ${\cal{W}}_{11}$ & $\{io_4\}$ & $\{io_8, io_9\}$ & ${\cal{P}}^{(11)}:$ (1300, 3, 65\%)\\
 \hline
 ${\cal{W}}_{12}$ & $\{io_4\}$ & $\{io_8, io_9\}$ & ${\cal{P}}^{(12)}:$ (900, 7, 83\%)\\
 \hline
 ${\cal{W}}_{13}$ & $\{io_4\}$ & $\{io_8, io_9\}$ & ${\cal{P}}^{(13)}:$ (400, 9, 93\%)\\
 \hline
 ${\cal{W}}_{14}$ & $\{io_4\}$ & $\{io_8, io_9\}$ & ${\cal{P}}^{(14)}:$ (750, 5, 79\%)\\
 \hline
 ${\cal{W}}_{15}$ & $\{io_5, io_6, io_7\}$ & $\{io_9, io_{10}, io_{11}\}$ & ${\cal{P}}^{(15)}:$ (700, 17, 91\%)\\
 \hline
 ${\cal{W}}_{16}$ & $\{io_5, io_6, io_7\}$ & $\{io_9, io_{10}, io_{11}\}$ & ${\cal{P}}^{(16)}:$ (500, 13, 90\%)\\
 \hline
 ${\cal{W}}_{17}$ & $\{io_8\}$ & $\{o_{12}\}$ & ${\cal{P}}^{(17)}:$ (150, 5, 86\%)\\
 \hline
 ${\cal{W}}_{18}$ & $\{io_8\}$ & $\{o_{12}\}$ & ${\cal{P}}^{(18)}:$ (400, 2, 73\%)\\
 \hline
 ${\cal{W}}_{19}$ & $\{io_8\}$ & $\{o_{12}\}$ & ${\cal{P}}^{(19)}:$ (300, 3, 81\%)\\
 \hline
 ${\cal{W}}_{20}$ & $\{io_9\}$ & $\{o_{13}\}$ & ${\cal{P}}^{(20)}:$ (1500, 12, 94\%)\\
 \hline
 ${\cal{W}}_{21}$ & $\{io_9\}$ & $\{o_{13}\}$ & ${\cal{P}}^{(21)}:$ (900, 14, 97\%)\\
 \hline
 ${\cal{W}}_{22}$ & $\{io_{10}\}$ & $\{o_{12}\}$ & ${\cal{P}}^{(22)}:$ (1700, 14, 87\%)\\
 \hline
 ${\cal{W}}_{23}$ & $\{io_9, io_{10}\}$ & $\{o_{14}\}$ & ${\cal{P}}^{(23)}:$ (1100, 10, 80\%)\\
 \hline
 ${\cal{W}}_{24}$ & $\{io_9, io_{10}\}$ & $\{o_{14}\}$ & ${\cal{P}}^{(24)}:$ (1700, 12, 81\%)\\
 \hline
 ${\cal{W}}_{25}$ & $\{io_{10}\}$ & $\{o_{12}\}$ & ${\cal{P}}^{(25)}:$ (1400, 13, 83\%)\\
 \hline
 ${\cal{W}}_{26}$ & $\{io_{10}\}$ & $\{o_{12}\}$ & ${\cal{P}}^{(26)}:$ (1900, 7, 80\%)\\
 \hline
 ${\cal{W}}_{27}$ & $\{io_{11}\}$ & $\{o_{13}\}$ & ${\cal{P}}^{(27)}:$ (1500, 11, 92\%)\\
 \hline
 ${\cal{W}}_{28}$ & $\{io_{11}\}$ & $\{o_{13}\}$ & ${\cal{P}}^{(28)}:$ (1100, 15, 94\%)\\
 \hline
 ${\cal{W}}_{29}$ & $\{io_{10}, io_{11}\}$ & $\{o_{15}\}$ & ${\cal{P}}^{(29)}:$ (500, 17, 72\%)\\
 \hline
 ${\cal{W}}_{30}$ & $\{io_{10}, io_{11}\}$ & $\{o_{15}\}$ & ${\cal{P}}^{(30)}:$ (350, 12, 74\%)\\
 \hline
\end{tabular}\label{tab:exampleServices}
\end{table}

Consider a query with inputs $i_1, i_2, i_3$ and desired outputs $o_{12}, o_{13}$. The objective is to find a solution to the query in such a way that the values of the QoS parameters are optimized (i.e., minimizing response time, maximizing throughput and reliability). It may be noted, a single solution may not be able to optimize all the QoS parameters. Therefore, multiple solutions need to be generated optimizing different QoS parameters.

The services that are eventually activated by the query inputs are shown in Figure \ref{fig:activation}. The services at $L_1$ of Figure \ref{fig:activation} are directly activated by the query inputs, while the services at $L_2$ and $L_3$ are indirectly activated by the query inputs. Each ellipse represents the input parameters available in the system at a particular point of time. Additionally, we have the following set of constraints.

\begin{itemize}
 \item ${\cal{LC}}_1:$ Each service participating in the solution must have a reliability value greater than $70\%$.
 \item ${\cal{GC}}_1:$ The reliability of the solution must be more than $60\%$.
 \item ${\cal{GC}}_2:$ The response time of the solution must be less than $2.5$s.
\end{itemize}

\begin{figure}[!htb]
\centering
\includegraphics[height=4cm,width=\linewidth]{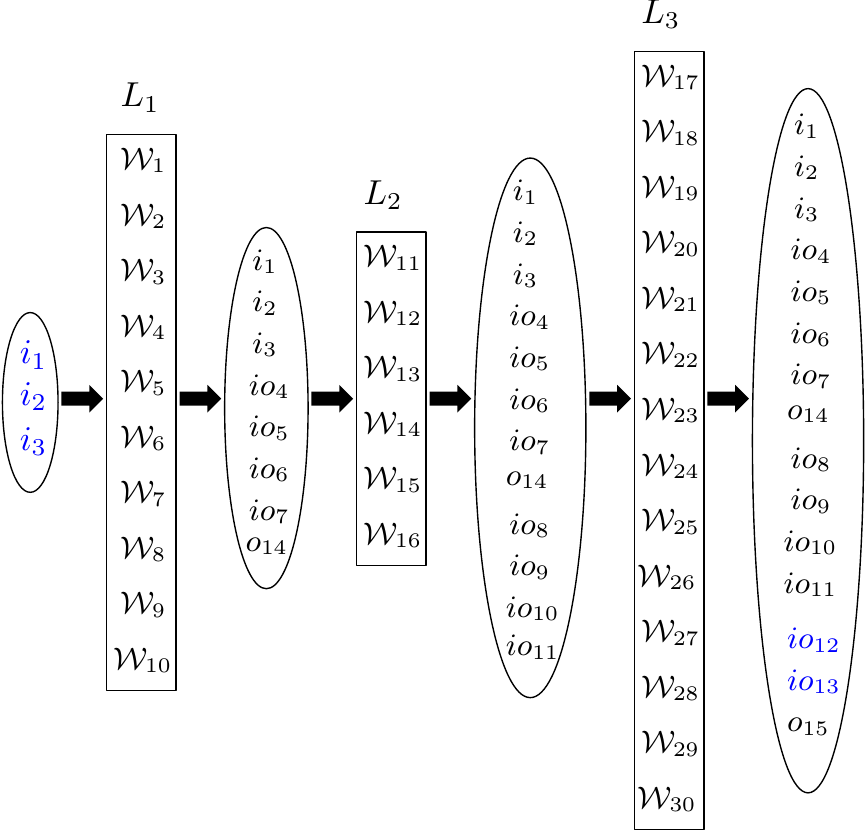}
\caption{Response to the query}
\label{fig:activation}
\end{figure}

\noindent
In this paper, we demonstrate the Pareto optimal solutions construction method given the above scenario using this example.
\hfill$\blacksquare$
\end{example}

%% file: Detailed_Methodology.tex
\section{Solution Architecture}\label{sec:method}
\noindent
In the following, we first define a few terminologies to build up the foundation of our work.

\input{definition}
\input{preprocessing}

\input{dependencyGraphConstruction}
\input{paretoOptimal}
\input{anytime}

\input{ga}

%% file: definition.tex
\begin{definition}{\em [\bf Dominating Service:]}
 A service ${\cal{W}}_i$ with QoS tuple ${\cal{P}}^{(i)} = ({\cal{P}}^{(i)}_1, {\cal{P}}^{(i)}_2, \ldots, {\cal{P}}^{(i)}_m)$ 
 dominates another service ${\cal{W}}_j$ with ${\cal{P}}^{(j)} = ({\cal{P}}^{(j)}_1, {\cal{P}}^{(j)}_2, \ldots, 
 {\cal{P}}^{(j)}_m)$, if $\forall$ $k$, ${\cal{P}}^{(i)}_k$ is at least as good as ${\cal{P}}^{(j)}_k$ and 
 $\exists$ $k$, such that, ${\cal{P}}^{(i)}_k$ is better than ${\cal{P}}^{(j)}_k$. ${\cal{W}}_i$ is 
 called the dominating service and ${\cal{W}}_j$ is dominated.
 \hfill$\blacksquare$
\end{definition}

\begin{example}
 Consider ${\cal{W}}_3$ and ${\cal{W}}_4$ in Table \ref{tab:exampleServices}
 with QoS tuples $(350, 4, 97\%)$ and $(475, 3, 85\%)$ respectively.
 ${\cal{W}}_3$ dominates ${\cal{W}}_4$, since ${\cal{W}}_3$ has a lesser response time,
 higher throughput and reliability as compared to ${\cal{W}}_4$.
 \hfill$\blacksquare$
\end{example}

\begin{definition}{\em [\bf Mutually Non Dominated Services:]}
 Two services ${\cal{W}}_i$ and ${\cal{W}}_j$ are said to be mutually non-dominated, if no one dominates the other, 
 i.e., no service is a dominating service.
 \hfill$\blacksquare$
\end{definition}

\begin{example}
 Consider ${\cal{W}}_1$ and ${\cal{W}}_3$ in Table \ref{tab:exampleServices}
 with QoS tuples $(500, 7, 93\%)$ and $(350, 4, 97\%)$ respectively.
 ${\cal{W}}_1$ and ${\cal{W}}_3$ are mutually non-dominated.
 ${\cal{W}}_1$ has higher throughput than ${\cal{W}}_3$, while ${\cal{W}}_3$ has 
 lower response time and higher reliability.
 \hfill$\blacksquare$
\end{example}

\begin{definition}{\em [\bf Skyline Service Set:]}
 Given a set of services ${\cal{WS}}$, the skyline service set ${\cal{WS}}^{*}$ is a subset of ${\cal{WS}}$ such that 
 the services in ${\cal{WS}}^{*}$ are non-dominated and each service in $({\cal{WS}} \setminus {\cal{WS}}^{*})$ 
 is dominated by at least one service in ${\cal{WS}}^{*}$. 
 \hfill$\blacksquare$
\end{definition}

\begin{example}
 Consider ${\cal{WS}}$ be the set $\{{\cal{W}}_1,{\cal{W}}_2, {\cal{W}}_3, {\cal{W}}_4\}$ with QoS tuples 
 $(500, 7, 93\%)$, $(600, 13, 69\%)$, $(350, 4, 97\%)$, $(475, 3, 85\%)$ respectively
 (as in Table \ref{tab:exampleServices}).
 ${\cal{WS}}^{*} = \{{\cal{W}}_1, {\cal{W}}_2, {\cal{W}}_3\} \subset {\cal{WS}}$ is the set of
 skyline services. ${\cal{W}}_1$, ${\cal{W}}_2$ and ${\cal{W}}_3$ are non-dominated, while ${\cal{W}}_3$ dominates ${\cal{W}}_4$. 
 \hfill$\blacksquare$
\end{example}

\noindent
The skyline service set for to a given set of services is unique.

\begin{definition}{\em [\bf Non Dominated Tuple:]}
 Given a set of QoS tuples ${\cal{TP}}$, a tuple $t \in {\cal{TP}}$ is called non dominated,
 if $\nexists$ $t' \in {\cal{TP}},$ such that $t'$ is better than $t$.
 \hfill$\blacksquare$
\end{definition}

\noindent
A QoS tuple $t'$ is {\em{better than}} $t$ implies, each QoS parameter in $t'$ is {\em{at least as good as}} in
$t$, while at least one QoS parameter in $t'$ is {\em{better than}} in $t$, where the terms {\em{``at least as good as"}} 
and {\em{``better than"}} are used with the same meaning as defined earlier in 
the context of comparing two services.

\begin{example}
 Consider a set of QoS tuples ${\cal{TP}} = \{(500, 7, 93\%), (600, 13, 69\%), (350, 4, 97\%), (475, 3, 85\%)\}$. 
 ${\cal{TP}}^{*}$ = $\{(500, 7, 93\%)$, $(600, 13, 69\%)$, $(350, 4, 97\%)\}$ constitute the set of non dominated tuples, 
 since, the three tuples in ${\cal{TP}}^{*}$ are non dominated and $(475, 3, 85\%) \in {\cal{TP}}$ is dominated by 
 $(350, 4, 97\%) \in {\cal{TP}}^{*}$.
 \hfill$\blacksquare$
\end{example}

\noindent
As discussed earlier, a composition solution with respect to a query ${\cal{Q}}$ is 
a collection of services that are eventually activated by ${\cal{Q}}^{ip}$ and produce ${\cal{Q}}^{op}$.
During this process of activation, all functional dependencies are preserved.
We now define different characterizations of a solution.

\begin{definition}{\em [\bf Feasible Solution:]}
 A composition solution is feasible if it satisfies all local 
 (${\cal{LC}}$) and global (${\cal{GC}}$) constraints.
 \hfill$\blacksquare$
\end{definition}

\begin{example}
 Consider the service descriptions and the query discussed in Example \ref{exm:overview}.
 $({\cal{W}}_2, {\cal{W}}_{11}, ({\cal{W}}_{18} || {\cal{W}}_{20}))$ is a solution to
 the query, where ${\cal{W}}_{18}$ and ${\cal{W}}_{20}$ are executed in parallel, while 
 ${\cal{W}}_{2}$, ${\cal{W}}_{11}$ and the parallel combination of ${\cal{W}}_{18}$ and 
 ${\cal{W}}_{20}$ are executed sequentially. The QoS tuple for the composition solution  
 is $(3400, 2, 31\%)$ {\small{$[RT: 600 + 1300 + Max(400, 1500) = 3400;$ 
 $T: Min(13, 3, 2, 12) = 2;$ $R: 69\% * 65\% * 73\% * 94\% = 31\%]$}}.
 This does not satisfy any of the global constraints and 
 ${\cal{W}}_{11}$ violates the local constraint as well and therefore, is not feasible. Consider another solution, 
 $({\cal{W}}_1, {\cal{W}}_{13}, ({\cal{W}}_{17} || {\cal{W}}_{21}))$. 
 The QoS tuple for the solution is $(1800, 5, 72.15\%)$, which satisfies both the
 local and the global constraints. Therefore, the solution is feasible.
 \hfill$\blacksquare$
\end{example}

\begin{definition}{\em [\bf Non Dominated Solution:]}
 A composition solution $S_i$ with QoS tuple ${\cal{P}}^{(S_i)} = ({\cal{P}}^{(S_i)}_1, {\cal{P}}^{(S_i)}_2, 
 \ldots, {\cal{P}}^{(S_i)}_m)$ is a non-dominated solution, if and only if 
 $\nexists$ $S_j$ with QoS tuple ${\cal{P}}^{(S_j)} = ({\cal{P}}^{(S_j)}_1, {\cal{P}}^{(S_j)}_2, 
 \ldots, {\cal{P}}^{(S_j)}_m)$ such that $\exists$ ${\cal{P}}^{(S_j)}_k \in {\cal{P}}^{(S_j)}$ 
 for which ${\cal{P}}^{(S_j)}_k$ is better than ${\cal{P}}^{(S_i)}_k$ and rest of the parameters in ${\cal{P}}^{(S_j)}$ 
 are at least as good as in ${\cal{P}}^{(S_i)}$. 
 \hfill$\blacksquare$
\end{definition}

\noindent
In other words, $S_i$ has a better value for at least one QoS ${\cal{P}}^{(S_i)}_k \in {\cal{P}}^{(S_i)}$ 
than any solution $S_j$, $S_j \ne S_i$.


\begin{definition}{\em [\bf Pareto Front:]}
 The set of non-dominated solutions with respect to a query is called the Pareto front.
 \hfill$\blacksquare$
\end{definition}

\noindent
In a multi-objective composition problem, we may not find a single solution which is optimal in all 
respects, rather, we may find a Pareto front consisting of a set of non-dominated solutions. The 
feasible solutions obtained from the Pareto front constitute the optimal solution space of our problem. 

We now present an optimal solution generation technique.
Our proposal has two main phases: a preprocessing phase and a run-time computation phase.
The aim of the preprocessing phase is to reduce the number of services participating in solution construction, 
while the main aim of the run-time computation phase is to compute the solution in response to a query.
Below, we explain our proposal in detail.

%% file: preprocessing.tex
\subsection{Preprocessing phase}
\noindent
The motivation behind preprocessing the web services is to reduce the run-time computation. We first define the notion of equivalent services, which serve as the foundation.

\begin{definition}{\em [\bf Equivalent Services:]}\label{def:equivalent}
 Two services ${\cal{W}}_i$ and ${\cal{W}}_j$ are equivalent (${\cal{W}}_i \simeq {\cal{W}}_j$),
 if the inputs of ${\cal{W}}_i$ are same as in ${\cal{W}}_j$, and the outputs of 
 ${\cal{W}}_i$ are same as in ${\cal{W}}_j$.
 \hfill$\blacksquare$
\end{definition}

\begin{example}
 Consider the first two services of Table \ref{tab:exampleServices}:
 ${\cal{W}}_1$ and ${\cal{W}}_2$ with input $\{i_1\}$ and outputs $\{io_4, io_5\}$.
 Here, ${\cal{W}}_1$ and ${\cal{W}}_2$ are equivalent, since, they have identical 
 inputs and outputs.
 \hfill$\blacksquare$
\end{example}




\noindent
Here, we apply the clustering technique proposed in \cite{7933195}. 
As the first step of preprocessing, we compute the set of equivalent services. Each equivalence class forms a cluster, while the set of equivalence classes of a given set of web services forms a partition of $W$. Therefore, the clusters are mutually exclusive and collectively exhaustive. We first divide the services in the service repository into multiple clusters and represent each cluster by a single representative service. Once the services are clustered, we find the skyline service set for each cluster. The set of skyline services is used for our run-time service composition step. 

The representative service corresponding to each cluster is associated with multiple QoS tuples corresponding to each service of the skyline service set. The main aim of this step is to prune the search space. Since the number of clusters must be less than or equal to the number of services in the service repository, the number of services reduces by this preprocessing. After preprocessing, we now have the following set of services: 
$W' = \{{\cal{W}}'_1, {\cal{W}}'_2, \ldots, {\cal{W}}'_{n'}\}$, where $n' \le n$ and each service ${\cal{W}}'_i \in W'$ consists of a set of QoS tuples ${\cal{P}}^{'(i)}$. Equality condition holds, only when the service repository does not contain any equivalent services and the preprocessing phase cannot reduce the number of services.


\begin{example}
 If we cluster the services shown in Table \ref{tab:exampleServices}, 
 the number of services reduces from 30 to 12.
 Table \ref{tab:exampleModifiedServices} shows the clustered service set.
 The first column of Table \ref{tab:exampleModifiedServices} presents the representative
 service for each cluster, while the second column shows the cluster itself. Finally,
 the third column indicates the set of QoS tuples corresponding to each service of the skyline
 services corresponding to a cluster. Consider the first cluster 
 $C_1: \{{\cal{W}}_1, {\cal{W}}_2, {\cal{W}}_3, {\cal{W}}_4\}$, shown in the first row of 
 Table \ref{tab:exampleModifiedServices}. ${\cal{W}}'_1$ is the representative service corresponding
 to $C_1$. The skyline service set of $C_1$ is $C^{*}_1: \{{\cal{W}}_1, {\cal{W}}_2, {\cal{W}}_3\}$.
 Therefore, ${\cal{P}}^{'(1)}$ consists of the QoS tuple corresponding to each service in $C^{*}_1$.
 
 \begin{table}[!ht]
\tiny
\caption{\scriptsize{Description of services after preprocessing}}
\centering
\begin{tabular}{c|c|c}
 Representative & Cluster & (RT, T, R)\\
 Web Service &  &  \\
 \hline\hline
 ${\cal{W}}'_1$ & $\{{\cal{W}}_1, {\cal{W}}_2, {\cal{W}}_3, {\cal{W}}_4\}$ & (500, 7, 93\%), (350, 4, 97\%), (600, 13, 69\%)\\
 \hline
 ${\cal{W}}'_2$ & $\{{\cal{W}}_5, {\cal{W}}_6\}$ & (700, 19, 90\%)\\
 \hline
 ${\cal{W}}'_3$ & $\{{\cal{W}}_7\}$ & (1100, 9, 80\%)\\
 \hline
 ${\cal{W}}'_4$ & $\{{\cal{W}}_8, {\cal{W}}_9, {\cal{W}}_{10}\}$ & (300, 13, 79\%)\\
 \hline
 ${\cal{W}}'_5$ & $\{{\cal{W}}_{11}, {\cal{W}}_{12}, {\cal{W}}_{13}, {\cal{W}}_{14}\}$ & (400, 9, 93\%)\\
 \hline
 ${\cal{W}}'_6$ & $\{{\cal{W}}_{15}, {\cal{W}}_{16}\}$ & (700, 17, 91\%), (500, 13, 90\%)\\
 \hline
 ${\cal{W}}'_7$ & $\{{\cal{W}}_{17}, {\cal{W}}_{18}, {\cal{W}}_{19}\}$ & (150, 5, 86\%)\\
 \hline
 ${\cal{W}}'_8$ & $\{{\cal{W}}_{20}, {\cal{W}}_{21}\}$ & (900, 14, 97\%)\\
 \hline
 ${\cal{W}}'_9$ & $\{{\cal{W}}_{22}, {\cal{W}}_{25}, {\cal{W}}_{26}\}$ & (1700, 14, 87\%), (1400, 13, 83\%)\\
 \hline
 ${\cal{W}}'_{10}$ & $\{{\cal{W}}_{23}, {\cal{W}}_{24}\}$ & (1100, 10, 80\%), (1700, 12, 81\%)\\
 \hline
 ${\cal{W}}'_{11}$ & $\{{\cal{W}}_{27}, {\cal{W}}_{28}\}$ & (1100, 15, 94\%)\\
 \hline
 ${\cal{W}}'_{12}$ & $\{{\cal{W}}_{29}, {\cal{W}}_{30}\}$ & (500, 17, 72\%), (350, 12, 74\%)\\
\end{tabular}\label{tab:exampleModifiedServices}
\end{table}

\noindent
Consider the query in Example \ref{exm:overview}. The number of services reduces from 30 to 12. The number of QoS tuples reduces from 30 to 18.\hfill$\blacksquare$

\end{example}

\noindent
The preprocessing step helps to prune the search space by removing some services.
No useful solution in terms of QoS values is lost in preprocessing, as stated
formally below.



\begin{lemma}\label{lemma:pareto}
 The preprocessing step is Pareto optimal solution preserving in terms of QoS values.
 \hfill$\blacksquare$
\end{lemma}

\noindent
All proofs are compiled in Appendix.

%% file: dependencyGraphConstruction.tex
\subsection{Dependency graph construction}
\noindent
The composition solutions are generated at run-time in response to a query. To find a response to a query, a dependency graph is constructed first. The dependency graph ${\cal{D}} = (V, E)$ is a directed graph, where $V$ is the set of nodes and $E$ is the set of edges. Each node $v_i \in V$ corresponds to a service ${\cal{W}}'_i \in W'$ that is eventually activated by the query inputs and each directed edge $(v_i, v_j) \in E$ represents a direct dependency between two services, i.e., the service corresponding to the node $v_i$ produces an output which is an input of the service corresponding to the node $v_j$. Each edge is annotated by the input-output of the services. Each solution to a query is either a path or a subgraph of ${\cal{D}}$ \cite{chattopadhyay2015scalable}.

The dependency graph is constructed using the algorithm illustrated in \cite{DBLP:journals/tweb/ChattopadhyayBB17}. While constructing the dependency graph, here we additionally validate the local and global constraints. While the local constraints are validated once, when a service is selected for the first time, the global constraints are validated in each step of the solution construction. While an activated service is selected for node construction, the service is first validated against the set of local and global constraints. Each service ${\cal{W}}'_i$ corresponds to a set of skyline services. If any service from the skyline services violates any local / global constraint, we disregard that service by removing its corresponding QoS tuple from ${\cal{P}}'_i$ of ${\cal{W}}'_i$. If ${\cal{P}}'_i$ is empty, we do not construct any node corresponding to ${\cal{W}}'_i$. It may be noted, if a service ${\cal{W}}'_i$ violates any of the global constraints, any solution that includes ${\cal{W}}'_i$ also violates the global constraint.

\begin{example}
 Consider Example \ref{exm:overview}. To respond to the query, while constructing the dependency graph, four services ${\cal{W}}'_1, {\cal{W}}'_2, {\cal{W}}'_3$ and ${\cal{W}}'_4$ are activated from the query inputs at first. It may be noted, ${\cal{W}}'_1$ is associated with three QoS tuples (500, 7, 93\%), (350, 4, 97\%) and (600, 13, 69\%), out of which one tuple (600, 13, 69\%)
 violates ${\cal{LC}}_1$, since its reliability is less than 70\%. Therefore, while validating ${\cal{LC}}_1$, the third tuple, i.e., (600, 13, 69\%) is removed from ${\cal{P}}^{'(1)}$ corresponding to ${\cal{W}}'_1$. 
 \hfill$\blacksquare$
\end{example}

\noindent
During dependency graph construction, the set of services that can be activated by the query inputs are identified first. With the set of identified services, the dependency graph is constructed. Finally, backward breadth first search (BFS) is used in ${\cal{D}}$ to identify the set of nodes that are required to produce the set of query outputs. The remaining nodes are removed from the graph.

\begin{figure}[!htb]
\centering
\includegraphics[width=\linewidth]{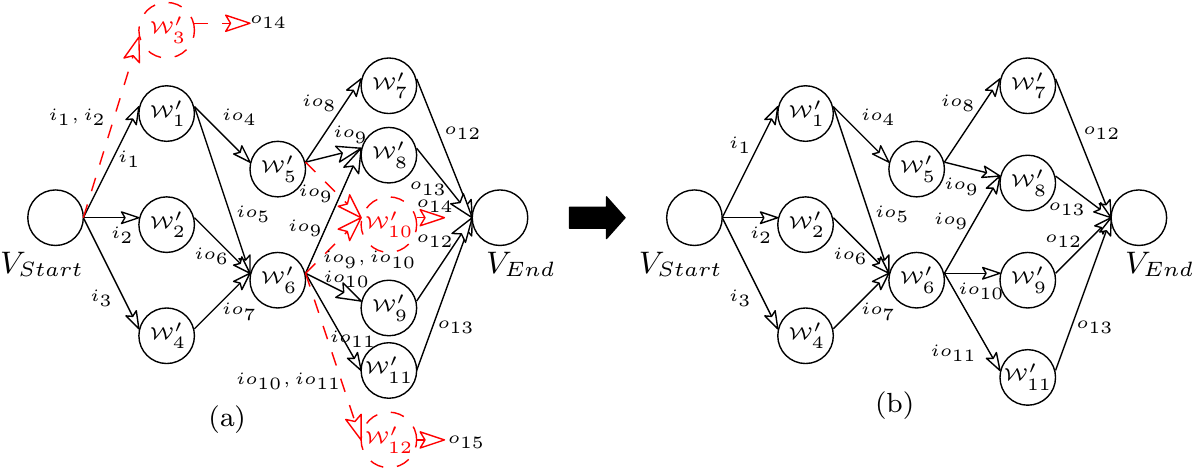}
\caption{Dependency Graph in response to a query (a) generated from query inputs (b) after backward traversal}
\label{fig:dependency}
\end{figure}

\begin{example}
 Consider the query in Example \ref{exm:overview}. Figure \ref{fig:dependency} shows the dependency graph constructed over the 
 services described in Table \ref{tab:exampleModifiedServices} in response to the query. Figure \ref{fig:dependency}(a) shows the dependency graph constructed from the query inputs, while Figure \ref{fig:dependency}(b) shows the one generated after removal of 
 unused nodes. In Figure \ref{fig:dependency}(a), the nodes marked with red represent the services that do not take part to produce the query outputs.
 \hfill$\blacksquare$
\end{example}

\noindent
If the dependency graph consists of a loop, we identify the loop and break the cycle \cite{rodriguez2015hybrid}. Finally, we partition the dependency graph into multiple layers using the approach used in \cite{DBLP:journals/tweb/ChattopadhyayBB17}, where a node $v_i$ belongs to a layer $L_k$, if for all the edges $(v_j, v_i)$, $v_j$ belongs to any layer $L_{k'}$, where $k' < k$. The first layer $L_0$ consists of a single node $V_{Start}$. Finally, we introduce dummy nodes in each layer if necessary as demonstrated in \cite{DBLP:journals/tweb/ChattopadhyayBB17}, to ensure that each node in a layer $L_i$ is connected only to the nodes in either its immediate predecessor layer $L_{(i - 1)}$ or its immediate successor layer $L_{(i + 1)}$. We assume that each dummy node has a QoS tuple with the best value for each QoS parameter. If a solution to a query consists of any dummy node, the dummy node is removed from the solution while returning the solution. The above assumption ensures that after removal of the dummy nodes, the QoS values of the solution remain unchanged. In the next subsection, we discuss the feasible Pareto optimal solution frontier generation technique.

%% file: paretoOptimal.tex
\section{Pareto Front Construction}
\noindent
\noindent
To find the feasible Pareto optimal solutions, we transform the dependency graph into a 
layered path generation graph (LPG). LPG $G_P = (V_P, E_P)$ is a 
directed acyclic graph, where $V_p$ is a set of nodes and $E_p$ is a set of edges.
Each node $v^{(p)}_i \in V_p$ consists of a set of nodes 
$\{v_{i_1}, v_{i_2}, \ldots, v_{i_l}\} \subseteq V$
of ${\cal{D}}$. A directed edge from $v^{(p)}_i$ to $v^{(p)}_j$ exists, if each 
service corresponding to a node $v_{j_k}$, belonging to $v^{(p)}_j$ is activated by the 
outputs of the services corresponding to the nodes, belonging to $v^{(p)}_i$.
Similar to the dependency graph, the LPG also consists of two dummy nodes: a start 
node $V^{(p)}_{Start}$ and an end node $V^{(p)}_{End}$
consisting of the start node and the end node of ${\cal{D}}$ respectively.
We assume that each dummy node has a QoS tuple with the best value for each QoS parameter. 
While constructing $G_P$, we simultaneously compute the Pareto optimal solution frontier 
and validate the global constraints. We define the notion of a cumulative Pareto optimal tuple. 

\begin{definition}{\em [\bf Cumulative Pareto Optimal Tuple:]}
 A set of non dominated QoS tuples, generated due to the composition of a set of services 
 during an intermediate step of the solution construction, is called a cumulative Pareto optimal tuple.
 \hfill$\blacksquare$
\end{definition}

\noindent
The cumulative Pareto optimal tuples, generated at the final step of the solution construction, 
is the Pareto front. For each node $v^{(p)}_i \in V_P$, we maintain two sets of 
QoS tuples: a set of non dominated tuples and a set of cumulative Pareto optimal tuples. 
We now discuss the construction of $G_P$.
 \begin{algorithm}
 \scriptsize
   \caption{Graph Conversion and Solution Generation}
   \begin{algorithmic}[1]
    \State {Input: ${\cal{D}} = (V, E)$, ${\cal{GC}}$}
    \State {Output: $G_P = (V_P, E_P)$}
    \State Queue $Q =$ Insert$(V^{(P)}_{End})$;
    \Repeat
     \State $v^{p}_i = $ Remove($Q$);\Comment{$v_i$ corresponding to ${\cal{W}}'_i$}
     \State ${\cal{PW}} \leftarrow$ the set predecessor nodes of $v^{p}_i$;\label{step:predecessor}
     \For {$v^{p}_j \in {\cal{PW}}$}
       \If {$v^{p}_j$ is not constructed earlier}
	\State ${\cal{P}}^{''(j)} \leftarrow$ the set of non dominated tuples corresponding to $v^{p}_j$;
	\If {$p \in {\cal{P}}^{''(j)}$ does not satisfy any global constraints}
	  \State Remove $p$ from ${\cal{P}}^{''(j)}$;
	\EndIf	
	\State {\bf{if}} {${\cal{P}}^{''(j)} = \phi$}, {\bf{then}} continue;
	\State $Q =$ Insert$(v^{p}_j)$;
       \EndIf
       \State Construct an edge $(v^{p}_j, v^{p}_i)$;
       \State ${\cal{CP}}^{(j)} \leftarrow$ The cumulative Pareto front of $v^{p}_j$ is constructed from \indent\indent \indent
	      (${\cal{CP}}^{(j)} \cup $ Combination of (${\cal{P}}^{''(j)}, {\cal{CP}}^{(i)}$));
       \If {$p \in {\cal{CP}}^{(j)}$ does not satisfy any global constraints}\label{step:constraintCheckStart}
	  \State Remove $p$ from ${\cal{CP}}^{(j)}$;
	\EndIf\label{step:constraintCheckEnd}
     \EndFor
    \Until{$(Q \ne \phi)$}
   \end{algorithmic}
   \label{algo:graphConversion}
 \end{algorithm}
  \begin{figure}[!htb]
  \centering
   \includegraphics[width=\linewidth]{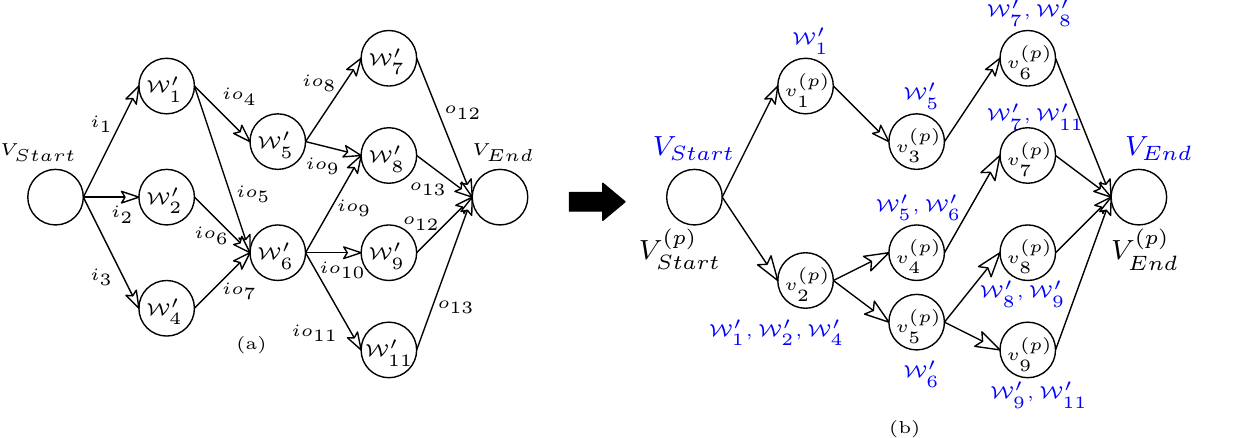}
  \caption{Conversion of Dependency Graph to LPG}
  \label{fig:conversion}
 \end{figure}
To construct $G_P$, we traverse the graph ${\cal{D}}$ in a backward direction, starting from 
the node $V_{End}$. We start the transformation from dependency graph to LPG
by constructing a dummy node 
$V^{(p)}_{End}$ of $G_P$ consisting of $V_{End}$ of ${\cal{D}}$. During the procedure, we maintain a
FIFO (i.e., First In First Out) queue. 
The following steps convert ${\cal{D}}$ to $G_P$:

\begin{figure*}[!htb]
  \centering
   \includegraphics[height=4cm,width=\linewidth]{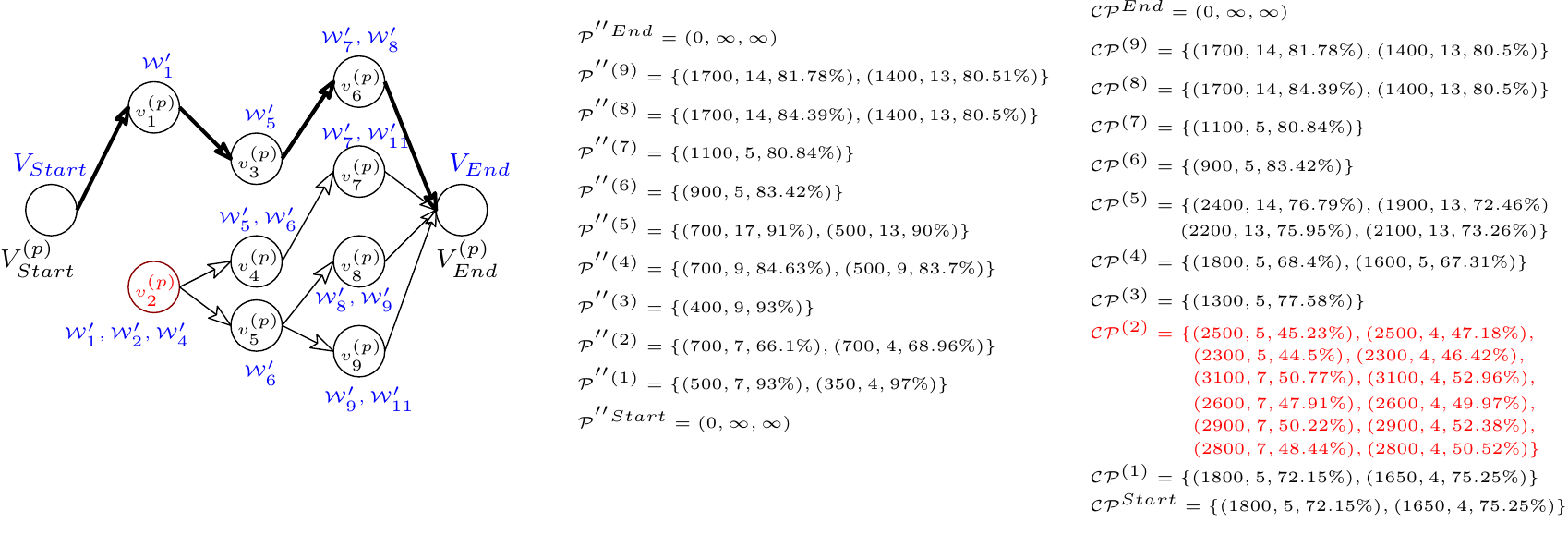}
  \caption{Pareto Optimal Front Construction}
  \label{fig:pareto}
 \end{figure*}

\begin{itemize}
 \item The first node $v^{{p}}_i$ is removed from the queue.
 \item The set of predecessor nodes of $v^{{p}}_i$ is constructed.
 \item For each predecessor node $v^{{p}}_j$ of $v^{{p}}_i$, 
	the temporary Pareto optimal front till $v^{{p}}_j$ is constructed or modified (for already existing nodes).
 \item For each Pareto optimal QoS tuple till $v^{{p}}_j$, the global constraints are validated.
	If any global constraint is violated, the tuple from the Pareto front is removed.
 \item Each predecessor node is inserted in the queue, if the queue does not already hold the same.
\end{itemize}

\noindent
We briefly elaborate each step below.
We first insert $V^{(p)}_{End}$ in the queue and then continue the procedure until the queue becomes empty.
In each step, we remove a node from the queue, say $v^{{p}}_i$ (in FIFO basis) and construct its predecessor 
nodes as described below.

Consider a node $v^{(p)}_i \in V_P$ consisting of a set of nodes $\{v_{i_1}, v_{i_2}, \ldots, v_{i_l}\} \subseteq V$
of ${\cal{D}}$. Also, consider $io_1, io_2, \ldots, io_x$ be the set of inputs that are required to activate the 
services corresponding to $\{v_{i_1}, v_{i_2}, \ldots, v_{i_l}\}$. For each $io_j \in \{io_1, io_2, \ldots, io_x\}$, 
we compute a set of nodes $V^{(io_j)} \subset V$, such that an edge $e \in E$ annotated by $io$ is incident to at least 
one node in $\{v_{i_1}, v_{i_2}, \ldots, v_{i_l}\}$. We then compute a set of combinations of nodes in ${\cal{D}}$ 
consisting of a node from each $V^{(io_j)}$, for all $io_j \in \{io_1, io_2, \ldots, io_x\}$. We now define 
the notion of a redundant service.

\begin{definition}{\em [\bf Redundant Service:]}
 A service ${\cal{W}}_k$ belonging to a solution $S_i$ in response to a query ${\cal{Q}}$ is 
 redundant, if $S_i \setminus \{{\cal{W}}_k\}$ is also a solution to ${\cal{Q}}$.
 \hfill$\blacksquare$
\end{definition}

\noindent
We consider the following assumption: a solution $S_i$ with some redundant services 
$\{{\cal{W}}_{i_1}, {\cal{W}}_{i_2}, \ldots, {\cal{W}}_{i_k}\}$ cannot be better, 
in terms of QoS values, than $S_i \setminus \{{\cal{W}}_{i_1}, {\cal{W}}_{i_2}, \ldots, {\cal{W}}_{i_k}\}$.
Two sets $V^{(io_j)}$ 
and $V^{(io_k)}$ may not be mutually exclusive, for $io_j, io_k \in \{io_1, io_2, \ldots, io_x\}$, since the service 
corresponding to one node may produce more than one output from $\{io_1, io_2, \ldots, io_x\}$. Therefore, if we consider 
a combination $c$ consisting of one node from each $V^{(io_j)}$, where 
$io_j \in \{io_1, io_2, \ldots, io_x\}$ , we do not need to
consider any combination which is a superset of $c$. For each combination, we construct a node $v^{(p)}_y$ and an edge 
$(v^{(p)}_y, v^{(p)}_i)$ of $G_P$.

\begin{example}
Fig. \ref{fig:conversion}(b) shows the LPG generated from 
 the dependency graph in Fig. \ref{fig:conversion}(a). Consider the end node $V^{(p)}_{End}$ 
 of Fig.\ref{fig:conversion}(b). $V^{(p)}_{End}$ consists of $V_{End}$. $\{o_{12}, o_{13}\}$ is the 
 required set of inputs (i.e., query outputs). $V^{(o_{12})} = \{{\cal{W}}'_7, {\cal{W}}'_9\}$ and 
 $V^{(o_{13})} = \{{\cal{W}}'_8, {\cal{W}}'_{11}\}$. We get 4 combinations from $V^{(o_{12})}, V^{(o_{13})}$
 and construct a node for each combination and the corresponding edges. 
 Consider another node of $G_P$ consisting of $\{{\cal{W}}'_7, {\cal{W}}'_8\}$. The required set of 
 inputs is $\{io_8, io_9\}$. $V^{(io_{8})} = \{{\cal{W}}'_5\}$ and $V^{(io_9)} = \{{\cal{W}}'_5, {\cal{W}}'_6\}$.
 We get 2 combinations from $V^{(io_8)}, V^{(io_9)}$. However, one combination $\{{\cal{W}}'_5, {\cal{W}}'_6\}$
 is a superset of another $\{{\cal{W}}'_5\}$. Therefore, we disregard the combination 
 $\{{\cal{W}}'_5, {\cal{W}}'_6\}$ and construct a node corresponding to $\{{\cal{W}}'_5\}$ and the corresponding edge. 
 \hfill$\blacksquare$  
\end{example}

\noindent
We now prove the following lemma.

\begin{lemma}
 Each path from $V^{(p)}_{Start}$ to $V^{(p)}_{End}$ in $G_P$ represents a solution to the query in terms 
 of functional dependencies.
 \hfill$\blacksquare$
\end{lemma}

\noindent
Once a node of $G_P$ is constructed, we construct the set of Pareto optimal tuples corresponding to the node.
Consider a node $v^{(p)}_i \in V_P$ consisting of a set of nodes 
$\{v_{i_1}, v_{i_2}, \ldots, v_{i_l}\} \subseteq V$ of ${\cal{D}}$. 
The QoS tuples corresponding to $\{v_{i_1}, v_{i_2}, \ldots, v_{i_l}\}$ 
are combined and a new set of tuples, ${\cal{P}}^{''(i)}$, is constructed. Each tuple in 
${\cal{P}}^{''(i)}$ is then validated against the set of global constraints ${\cal{GC}}$. 
If any tuple violates any of the global constraints,
the tuple is removed from ${\cal{P}}^{''(i)}$. If no tuple from ${\cal{P}}^{''(i)}$ satisfies the 
global constraints, we disregard the node $v^{(p)}_i$. 
Otherwise, we compute the set of non dominated tuples from ${\cal{P}}^{''(i)}$ and associate these with $v^{(p)}_i$. 
Consider $v^{(p)}_j$ is removed from the queue and $v^{(p)}_i$ is created as the predecessor 
of $v^{(p)}_j$. If $v^{(p)}_i$ already exists in the queue, we do not need to recompute the set of non dominated tuples
of $v^{(p)}_i$.
Once the set of non dominated tuples corresponding to $v^{(p)}_i$ are constructed, we construct the cumulative 
Pareto optimal solutions till $v^{(p)}_i$. 

In order to find the Pareto front till $v^{(p)}_i$, we combine the tuples in the Pareto front 
constructed till $v^{(p)}_j$ with the set of non dominated tuples of $v^{(p)}_i$. The combined tuples are verified
against the global constraints and if any tuple violates any of the global constraints, we remove the 
tuple from the combined set. Finally, we compute the cumulative Pareto optimal solutions till $v^{(p)}_i$ from 
the set of combined tuples and the cumulative Pareto front of $v^{(p)}_i$. The Pareto front constructed in 
$V^{(p)}_{Start}$ constitutes the feasible Pareto optimal solutions to the query.

\begin{example}
 Fig. \ref{fig:pareto} shows the Feasible Pareto front generation method on a LPG. 
 The set of initial non dominated tuples of $V^{(p)}_{End}$ consists of one tuple $(0, \infty, \infty)$, initialized with the best values of these parameters.
 The cumulative Pareto optimal front of $V^{(p)}_{End}$ also consists of the same tuple. 
 
 Now consider a node $v^{(p)}_5$. When $v^{(p)}_5$ is created as a predecessor of $v^{(p)}_8$, 
 the set of non dominated tuples ${\cal{P}}^{''(5)}$ corresponding to $v^{(p)}_5$ are constructed first. 
 The cumulative Pareto front ${\cal{CP}}^{(5)}$ till $v^{(p)}_5$ is constructed next by combining 
 the cumulative Pareto front till $v^{(p)}_8$ and the set of non dominated tuples of $v^{(p)}_5$, followed by 
 selecting the Pareto front from the combined set. In the next iteration, when $v^{(p)}_5$ is 
 constructed as a predecessor of $v^{(p)}_9$, the set of non dominated tuples are not recomputed. However, 
 ${\cal{CP}}^{(5)}$ is modified. The cumulative Pareto front till $v^{(p)}_9$ and the set of non dominated tuples 
 of $v^{(p)}_5$ are combined first and then the Pareto front is selected from the combined set and 
 the already existing set ${\cal{CP}}^{(5)}$. It may be noted, the cumulative Pareto front till $v^{(p)}_2$
 violates the global constraints. Hence, the node is disregarded from the graph.
 The final solution path is marked by the bold line.
 \hfill$\blacksquare$ 
\end{example}

\noindent
Algorithm \ref{algo:graphConversion} presents the formal algorithm for constructing the feasible Pareto 
optimal solution in response to a query. We now prove the following lemma.

\begin{lemma}\label{lemma:complete}
 Algorithm \ref{algo:graphConversion} is complete.
 \hfill$\blacksquare$
\end{lemma}

\begin{figure*}[!htb]
  \centering
   \includegraphics[height=3.5cm,width=\linewidth]{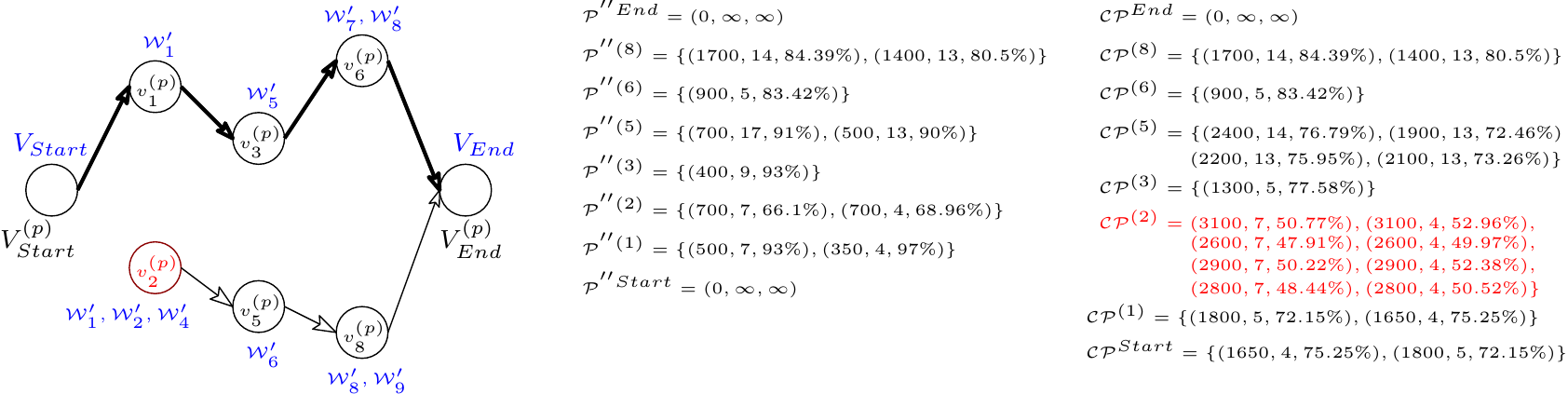}
  \caption{Heuristic solution Construction}
  \label{fig:heuristic}
 \end{figure*}

\begin{lemma}\label{lemma:sound}
 Algorithm \ref{algo:graphConversion} is sound.
 \hfill$\blacksquare$
\end{lemma}

\noindent
The search space of this algorithm is exponential in terms of the number 
of services required to serve a query. This limits its scalability to large service repositories. 
In the next subsection, we propose two scalable heuristics. 

%% file: anytime.tex
\section{A Heuristic Approach}
\noindent
We first discuss the limitation of the solution discussed in the previous subsection. It is easy to see that
Step \ref{step:predecessor} of Algorithm \ref{algo:graphConversion}, where the set of predecessors of a node 
is constructed, may explode. Consider the following example.
\begin{example}
Consider a node $v^{p}_i$ of LPG requires 10 inputs and each input is provided by 10 nodes of the 
dependency graph. The number of possible predecessor nodes of $v^{p}_i$ is $10^{10}$. 
 \hfill$\blacksquare$
\end{example}

\noindent
If the number of inputs of a node or the number of nodes providing an input, increases, the number 
of predecessor nodes also increases exponentially.
In our heuristic, we try to address the above issue. The main motivation of this algorithm is 
to reduce the search space of the original problem. On one hand, we attempt to reduce the number of combinations
generated at Step \ref{step:predecessor} of Algorithm \ref{algo:graphConversion}. On the other hand, we try
to restrict the number of nodes generated at a particular level of the LPG. 
Our approach is based on the notion of anytime algorithms\cite{yan2015anytime} using beam search. Beam search uses breadth-first search to build its search space. However, at each level of the graph, we store only a fixed number of nodes, called the beam width. The greater the beam width, the fewer the number of nodes pruned. 

While constructing $G_P$, all predecessor nodes of the set of nodes at a particular level are computed, as earlier. However, only a subset of the nodes is stored depending on the beam width of the algorithm. Consider $\{v^{(p)}_{i_1}, v^{(p)}_{i_2}, \ldots, v^{(p)}_{i_k}\} \subset V_P$ be the set of nodes generated at level $i$ and the beam width is $k_1 \le k$. Therefore, only $k_1$ out of $k$ nodes are stored. The nodes are selected based on the values of the cumulative Pareto optimal tuples computed till $v^{(p)}_{i_j}$, for $j = 1, 2, \ldots, k$. 
At each level, the selected set of nodes is ranked between $1, 2, \ldots, k_1$. A node with rank $i$ has higher priority than a node with rank $j$, where $j > i$. Consider $l$ be the number of levels in the dependency graph, where the level $l$ consists of $V_{End}$. The selection criteria for choosing $k_1$ nodes from $(l - 1)^{th}$ level is discussed below.

\begin{itemize}
 \item The feasible non dominated tuples $V_{(l-1)}^*$ corresponding to the  cumulative Pareto optimal tuples 
	computed till each $v^{(p)}_{{(l-1)}_j}$, for $j = 1, 2, \ldots, k$ are computed first.
 \item If $|V_{(l-1)}^*| = k_1$, $V_{(l-1)}^*$ is returned.
 \item If $|V_{(l-1)}^*| > k_1$, the following steps are performed:
 \begin{itemize}
  \item The utility $U(t_j)$ corresponding to each tuple $t_j \in V_{(l-1)}^*$
	is computed as follows:
  {\scriptsize{
  \begin{equation}
    A_k = Max_{t_j \in V_{(l-1)}^*}{\cal{P}}_k^{(t_j)}
  \end{equation}
  \begin{equation}
    B_k = Min_{t_j \in V_{(l-1)}^*}{\cal{P}}_k^{(t_j)}
  \end{equation}
  \begin{equation}
    NV({\cal{P}}_k^{(t_j)}) = \left\{
  \begin{array}{ll}
      \frac{A_k - {\cal{P}}_k^{(t_k)}}{A_k - B_k} & ,\text{for a negative QoS} \\
      \frac{{\cal{P}}_k^{(t_j)} - B_k}{A_k - B_k} & ,\text{for a positive QoS} \\
      1 & ,A_k = B_k \\
\end{array} 
\right.
\end{equation}
\begin{equation}\label{equ:utility}
   U(t_j) = \sum_{{\cal{P}}_k \in t_j} NV({\cal{P}}_k^{(t_j)})
\end{equation}
}}where $NV({\cal{P}}_k^{(t_j)})$ is the normalized value of ${\cal{P}}_k$ for $t_j \in V_{(l-1)}^*$.

  \item Tuples in $V_{(l-1)}^*$ are sorted in descending order based on utility and finally, the first $k_1$ 
	tuples are chosen, which are returned.
 \end{itemize}
  \item If $|V_{(l-1)}^*| < k_1$, $(k_1 - |V_{(l-1)}^*|)$ more tuples are chosen from $\{v^{(p)}_{{(l-1)}_1}, v^{(p)}_{{(l-1)}_2}, \ldots, v^{(p)}_{{(l-1)}_k}\} \setminus |V_{(l-1)}^*|$ using the same procedure discussed above.
  \item Finally, the selected tuples are ranked based on their utility value. The rank of the tuple with the highest utility value is set to 1.
\end{itemize}

\noindent
Consider ${\cal{TP}}_{(i + 1)} = \{t_{(i + 1)_1}, t_{(i + 1)_2}, \ldots, t_{(i + 1)_{k'}}\}$ be the set of tuples selected at level $(i + 1)$ of $G_P$. It may be noted, $|{\cal{TP}}_{(i + 1)}| = k' \le k_1$. Without loss of generality, we assume that the priority of $t_{(i + 1)_{j1}}$ is greater than that of $t_{(i + 1)_{j2}}$, where $j_1 < j_2$. Also consider $\{v^{(p)}_{(i + 1)_1}, v^{(p)}_{(i + 1)_2}, \ldots, v^{(p)}_{(i + 1)_{k''}}\} \subset V_P$ be the set of nodes corresponding to the tuples in ${\cal{TP}}_{(i + 1)}$. We further assume that at least one tuple corresponding to $v^{(p)}_{(i + 1)_{j1}}$ has higher priority than any tuple corresponding to $v^{(p)}_{(i + 1)_{j2}}$, where $j_1 < j_2$. 
We now discuss the selection criteria for choosing $k_1$ nodes from level $i$, where $i < (l - 1)$:
\begin{itemize}
 \item A cumulative Pareto optimal tuple $t$ corresponding to a node $v$ is selected for the rank position 1, if the following conditions are satisfied:
 \begin{itemize}
  \item $v$ belongs to the set of predecessor nodes of $v^{(p)}_{(i + 1)_1}$.
  \item $t$ has the highest utility value among all the cumulative Pareto optimal tuples corresponding to all the predecessor nodes of $v^{(p)}_{(i + 1)_1}$. 
 \end{itemize}
 \item In general, a cumulative Pareto optimal tuple $t$ corresponding to a node $v$ is selected for the rank position $j$, where $j \le k'$, if:
 \begin{itemize}
  \item $v$ belongs to $\cup_{x=1}^{j} Pred(Corr(t_{(i + 1)_x}))$, where $Corr(t_{(i + 1)_x})$ is the node corresponding to $t_{(i + 1)_x}$ and $Pred(Corr(t_{(i + 1)_x}))$ is the set of predecessor nodes of $Corr(t_{(i + 1)_x})$.
  \item $t$ has the highest utility value among all the cumulative Pareto optimal tuples corresponding to $\cup_{x=1}^{j} Pred(Corr(t_{(i + 1)_x})) \setminus \{\text{set of already selected tuples}\}$.
 \end{itemize}
 \item A cumulative Pareto optimal tuple $t$ corresponding to a node $v$ is selected for rank $j$ ($k' \le j \le k_1$), if:
 \begin{itemize}
  \item $t$ has the highest utility value among all the cumulative Pareto optimal tuples corresponding to $\cup_{x=1}^{k''} Pred(v^{(p)}_{(i + 1)_x}) \setminus \{\text{set of already selected tuples}\}$.
 \end{itemize}
\end{itemize}

\noindent
Clearly, the search space of the algorithm is determined by the beam width.
The main motivation of this algorithm is to be able to improve the solution quality monotonically with the increase in beam width. The selection procedure enforces this criteria.
As the beam size increases, the number of pruned nodes decreases and the solution quality of the algorithm
either remains same or improves, as formally stated below.

\begin{lemma}
The solution quality of the heuristic algorithm monotonically improves with increase in beam width.
 \hfill$\blacksquare$
\end{lemma}


\begin{lemma}\label{lemma:AnytimeToOptimal}
With an infinite beam width, the algorithm is identical to the Pareto optimal algorithm.
\hfill$\blacksquare$
\end{lemma}

\noindent
Since at each level, a finite number of nodes are generated, the above lemma holds. 
Moreover, if the beam width of the heuristic algorithm is greater than or equal to the 
maximum number of nodes belonging to a level of $G_P$, where $G_P$ represents the complete LPG 
constructed by the optimal algorithm, no nodes are required to be pruned in this algorithm 
and thereby, the algorithm is identical to the optimal one.

%% file: ga.tex
\section{Solution Generation Using NSGA}\label{sec:ga}
\noindent
In this section, we present a different approach based on the non-dominated sorting genetic algorithm (NSGA) \cite{deb2002fast}. While the previous algorithms are deterministic algorithms, this is a randomized algorithm. This algorithm is basically an adaptive heuristic search algorithm based on the evolutionary ideas of natural selection and genetics \cite{srinivas1994genetic}. 

In this algorithm, a population of candidate solutions, called phenotypes \cite{srinivas1994genetic}, to a query is evolved toward better solutions. Each candidate solution is encoded into a binary string, called chromosome or genotype, which can be mutated and altered. The algorithm starts from a population of randomly generated chromosomes. In each iteration (called a generation) of the algorithm, a new set of chromosomes are generated and the fitness of every chromosome in the population is evaluated. The fitness value of a chromosome usually refers to the value of the objective function in the optimization problem being solved. In this paper, the fitness value of a chromosome is computed (as discussed later) from the QoS values of its corresponding phenotype. The more fit chromosomes are selected from the current population to form the next generation. Each chromosome in the current population is modified using different genetic operators (namely crossover, mutation) to construct its off-string for the next generation. The algorithm terminates after a fixed number of iterations, which is provided externally.

Once a query comes to the system, the dependency graph is constructed first. The dependency graph contains all possible solutions to the query. The genetic algorithm is applied on the dependency graph to generate high-quality solutions. Before discussing the details of the algorithm, we first define the notion of a chromosome used in this paper.

\begin{definition}{\em [\bf Chromosome:]}
 A chromosome is a binary string ${\cal{B}} = b_1 b_2 \ldots b_n$, where each bit $b_i \in \{b_1, b_2, \ldots, b_n\}$ of the string represents a node $v_i \in V$ in the dependency graph ${\cal{D}} = (V, E)$ constructed in response to a query and is defined as:
 $b_i = \left\{
 \begin{array}{ll}
  1 & \text{if } v_i \text{ is present in the solution sub-graph }\\
  0 & \text{otherwise }
 \end{array}
 \right. $
 \hfill$\blacksquare$
\end{definition}

\noindent
It may be noted, each binary string is not a valid solution in terms of functional dependency (e.g., the string of all $0$'s). Therefore, in our algorithm, we consider only those chromosomes which represent a valid solution to a query in terms of functional dependency. Later in this section, we present an algorithm for constructing a chromosome. In the next subsection, we present an overview of our algorithm.

\subsection{Algorithm to generate Pareto optimal solution}
\noindent
Our algorithm has the following steps. 

\begin{enumerate}
 \item {\bf{Initial population construction}}: First $n$ chromosomes are chosen randomly to construct the initial population.
 \item {\bf{Iteration of the algorithm}}: In each iteration of the algorithm, a new generation is created. The algorithm terminates after a fixed number of iterations $k$, where $k$ is provided externally. The following steps are performed in each iteration of the algorithm.
 \begin{enumerate}
  \item \label{step:selection} {\bf{Selection of parent chromosomes}}: Two chromosomes are chosen randomly from the current population using Roulette wheel selection \cite{srinivas1994genetic}. 
  \item \label{step:crossover} {\bf{Crossover}}: With probability $p_c$ (i.e., crossover probability), the crossover operation is performed between the two chromosomes selected in the earlier step to generate two off-springs. Steps \ref{step:selection} and \ref{step:crossover} are performed $n$ number of times.
  \item {\bf{Mutation}}: Each off-spring, generated in the earlier step, is mutated to generate a new off-spring.
  \item {\bf{Fitness value computation}}: The fitness value for each chromosome in the population is computed. 
  \item {\bf{Construction of new generation}}: The population now consists of more than $n$ chromosomes. Based on the fitness value of the chromosomes in the population, the best $n$ chromosomes are selected from the population and the rest of the chromosomes are removed.
 \end{enumerate}
 \item {\bf{Solution construction}}: Once the algorithm terminates, the best chromosomes, obtained from the final population, are returned as the solutions of our algorithm.
\end{enumerate}

\noindent
 We now briefly demonstrate each step of this algorithm.

\subsection{Chromosome Construction}
\noindent
To construct a chromosome, we traverse the dependency graph backward starting from $V_{End}$. Each bit of a chromosome is first initialized by 0. During the traversal of the dependency graph, whenever a node of the dependency graph is visited, the corresponding bit in the chromosome is marked as 1. 

\begin{itemize}
 \item We start from $V_{End}$ and mark the corresponding bit in chromosome as 1.
 \item For each input $io$ of a visited node $v_i$, we randomly select a node $v_j$ that produces $io$ as output and mark the corresponding bit for $v_j$ as 1.
 \item Each node is processed only once. In other words, if a node $v_i$ is encountered more than once, we do not reprocess the node.
 \item The algorithm terminates when there is no node left for processing.
\end{itemize}

\noindent
It may be noted, once the algorithm terminates, we get a valid solution in terms of functional dependency corresponding to the generated chromosome. 

\begin{example}
 Consider the example shown in Figure \ref{fig:chromosome}. Figure \ref{fig:chromosome}(a) shows a dependency graph constructed in response to a query, while Figure \ref{fig:chromosome}(b) presents an example candidate solution (i.e., phenotype) constructed from the dependency graph. 
 \begin{figure}[!htb]
  \centering
   \includegraphics[width=\linewidth]{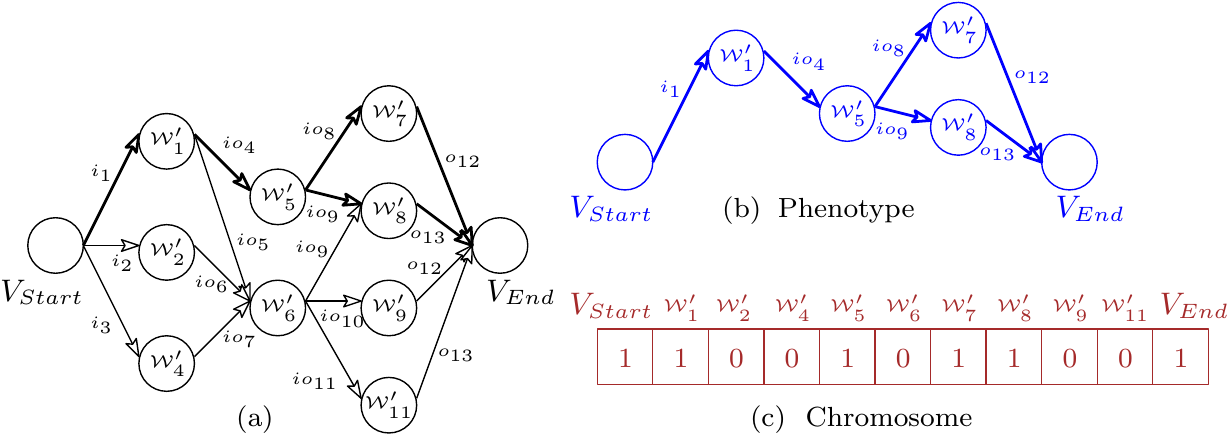}
  \caption{Example of a chromosome}
  \label{fig:chromosome}
 \end{figure}
 As shown in the figure, $V_{End}$ takes two inputs $o_{12}$ and $o_{13}$. For each input, one node is randomly selected (as shown in the figure by bold line). The algorithm terminates when there is no node left for processing. Figure \ref{fig:chromosome}(c) represents the chromosome corresponding to the phenotype shown in Figure \ref{fig:chromosome}(b).
 \hfill$\blacksquare$
\end{example}

\noindent
We now discuss the two main operators of the genetic algorithm, namely, crossover and mutation.

\subsection{Crossover}
\noindent
Crossover is a binary genetic operator used to obtain chromosomes from one generation to the next. In this paper, we consider only single point crossover \cite{srinivas1994genetic}. However, the crossover mechanism applied here is different from the normal single point crossover, in order to ensure that the off-springs generated after the crossover operation are valid solutions in terms of functional dependency.

Consider two parent chromosomes $CR_1$ and $CR_2$ selected to participate in a crossover. Also consider $D_{CR1} = (V_{CR1}, E_{CR1})$ and $D_{CR2} = (V_{CR2}, E_{CR2})$ are two subgraphs (i.e., phenotypes) of the dependency graph $D$ corresponding to $CR_1$ and $CR_2$ respectively.

The main intuition of this step is to randomly choose a common node $v_i$ between $D_{CR1}$ and $D_{CR2}$ at first and then exchange the set of nodes in $D_{CR1}$ and $D_{CR2}$ that are responsible to activate $v_i$ for both $D_{CR1}$ and $D_{CR2}$. In this way we ensure to obtain functionally valid off-springs. We now formally demonstrate the crossover operation.

\begin{itemize}
 \item At first, the set of common nodes from $D_{CR1}$ and $D_{CR2}$ are identified. Out of them, one common node $v_i$ is chosen randomly. 
 \item For each of the subgraphs $D_{CR1}$ and $D_{CR2}$, we compute the subgraphs $D'_{CR1} = (V'_{CR1}, E'_{CR1})$ containing all the paths from $V_{Start}$ to $v_i$ of $D_{CR1}$ and $D'_{CR2} = (V'_{CR2}, E'_{CR2})$ containing all the paths from $V_{Start}$ to $v_i$ of $D_{CR2}$.
 \item Finally, $D'_{CR1}$ and $D'_{CR2}$ are exchanged to generate two new off-springs $D_{CR3} = (V_{CR3}, E_{CR3})$ and $D_{CR4} = (V_{CR4}, E_{CR4})$, where, \\
 {\small $V_{CR3} = (V_{CR1} \setminus \{v\}) \cup V'_{CR2}$}, where $v \in V'_{CR1}$ and $v$ does not belong to any path other than $V_{Start}$ to $v_i$ in $E_{CR1}$;\\
 {\small $E_{CR3} = (E_{CR1} \setminus \{e\}) \cup E'_{CR2}$}, where $e \in E'_{CR1}$ and $e$ does not belong to any path other than $V_{Start}$ to $v_i$ in $E_{CR1}$; \\
 {\small $V_{CR4} = V'_{CR1} \cup (V_{CR2} \setminus \{v\})$}, where $v \in V'_{CR2}$ and $v$ does not belong to any path other than $V_{Start}$ to $v_i$ in $E_{CR2}$; \\
 {\small $E_{CR4} = E'_{CR1} \cup (E_{CR2} \setminus \{e\})$}, where $e \in E'_{CR2}$ and $e$ does not belong to any path other than $V_{Start}$ to $v_i$ in $E_{CR2}$.
\end{itemize}

\noindent
Each off-spring is generated after scanning two subgraphs $D_{CR1}$ and $D_{CR2}$ once. For example, during the construction of $D_{CR3}$, $D_{CR1}$ is traversed backward starting from $V_{End}$. Once a node of $D_{CR1}$ is encountered, the node is copied in $D_{CR3}$. For each node $v_j$ (other than $v_i$) in $D_{CR3}$, all the nodes that provide at-least one input to $v_j$ in $D_{CR1}$ are copied in $D_{CR3}$. Once the traversal of $D_{CR1}$ is done, the traversal of $D_{CR2}$ starts backward starting from $v_i$. Similar to the previous, once a node of $D_{CR2}$ is encountered, the node is copied in $D_{CR3}$. For each node $v_j$ in $D_{CR3}$, all the nodes that provide at-least one input to $v_j$ in $D_{CR2}$ are copied in $D_{CR3}$.

\begin{example}
Consider two subgraphs of the dependency graph corresponding to two parent chromosomes as shown in Figure \ref{fig:crossover}(a) and (b). The node corresponding to ${\cal{W}}'_{12}$ is chosen as the common node. Finally, the off-springs generated using the above procedure are shown in Figure \ref{fig:crossover}(c) and (d).
 \begin{figure}[!htb]
  \centering
   \includegraphics[width=\linewidth]{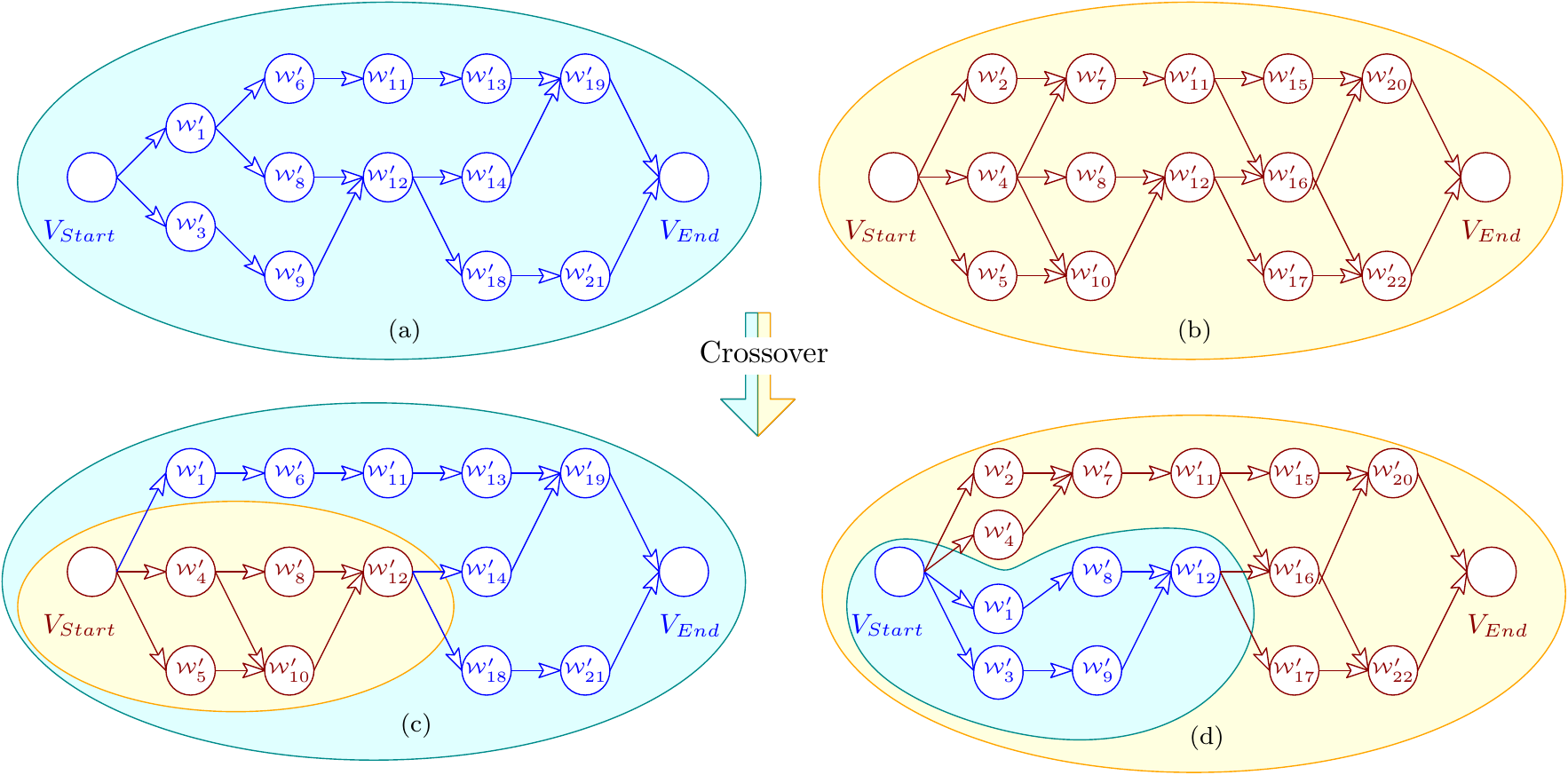}
  \caption{Example of crossover}
  \label{fig:crossover}
 \end{figure}
 \hfill$\blacksquare$
\end{example}

 
\subsection{Mutation}
\noindent
Mutation is a unary genetic operator used to obtain a chromosome from one generation to the next. Like crossover, mutation is also performed in a different way than normal mutation operation in our approach to ensure functional dependency preserving solutions are generated.

Consider a parent chromosome $CR_1$ to participate in mutation. Also consider $D_{CR1} = (V_{CR1}, E_{CR1})$ is the subgraph of the dependency graph $D$ corresponding to $CR_1$ and $p_m$ is the mutation probability. 

The intuition of this step is as follows. We traverse $D_{CR1}$ in backward. While traversing, with probability $p_m$ we choose a node $v_i$ to be mutated. If a node $v_j$ is selected instead of $v_i$ after mutation, we randomly generate the nodes that are needed to activate $v_j$ from the query inputs. In this manner, we ensure to obtain a functionally valid off-spring. We now formally discuss the mutation operation.

\begin{itemize}
 \item We start traversing $D_{CR1}$ backward starting from $V_{End}$ to generate a new off-spring $D_{CR2} = (V_{CR2}, E_{CR2})$. We first add $V_{End}$ in $V_{CR2}$.
 \item For each input $io$ of a node $v_i \in V_{CR2}$, if there is only one node $v_j \in V$ in the dependency graph $D$ to produce $io$ as output, we add $v_j$ in $V_{CR2}$ and $(v_j, v_i)$ in $E_{CR2}$.
 \item For each input $io$ of a node $v_i \in V_{CR2}$, if $v_i \in V_{CR1}$ and there exists multiple nodes in the dependency graph $D$ to produce $io$ as output, with probability $(1 - p_m)$ we add $v_j$ in $V_{CR2}$, where $(v_j, v_i) \in E_{CR1}$ and $v_j$ produces $io$ as output. We also add $(v_j, v_i)$ in $E_{CR2}$.
 \item For each input $io$ of a node $v_i \in V_{CR2}$, if $v_i \in V_{CR1}$ and there exists multiple nodes in the dependency graph $D$ to produce $io$ as output, with probability $p_m$, we perform the following operation: we randomly choose a node $v_j$ from $D$ such that $(v_j, v_i) \notin E_{CR1}$ and $v_j$ produces $io$. We add $v_j$ to $V_{CR2}$ and $(v_j, v_i)$ to $E_{CR2}$.
 \item For each input $io$ of a node $v_i \in V_{CR2}$, if $v_i \notin V_{CR1}$, we randomly select a node $v_j$ that produces $io$ as output and add $v_j$ to $V_{CR2}$ and $(v_j, v_i)$ to $E_{CR2}$.
 \item Each node in $V_{CR2}$ is processed only once. In other words, if a node $v_i$ is encountered more than once, we do not reprocess the node.
 \item The algorithm terminates when there is no node left for processing.
\end{itemize}

\begin{example}
 Consider the dependency graph shown in Figure \ref{fig:mutation}(a) and the subgraph of the dependency graph corresponding to a chromosome shown in Figure \ref{fig:mutation}(b). In the figure, for input $io_9$ of the node corresponding to ${\cal{W}}'_8$, the mutation operation is performed. Finally, the off-spring generated using the above procedure is shown in \ref{fig:mutation}(c).
 \begin{figure}[!htb]
  \centering
   \includegraphics[width=\linewidth]{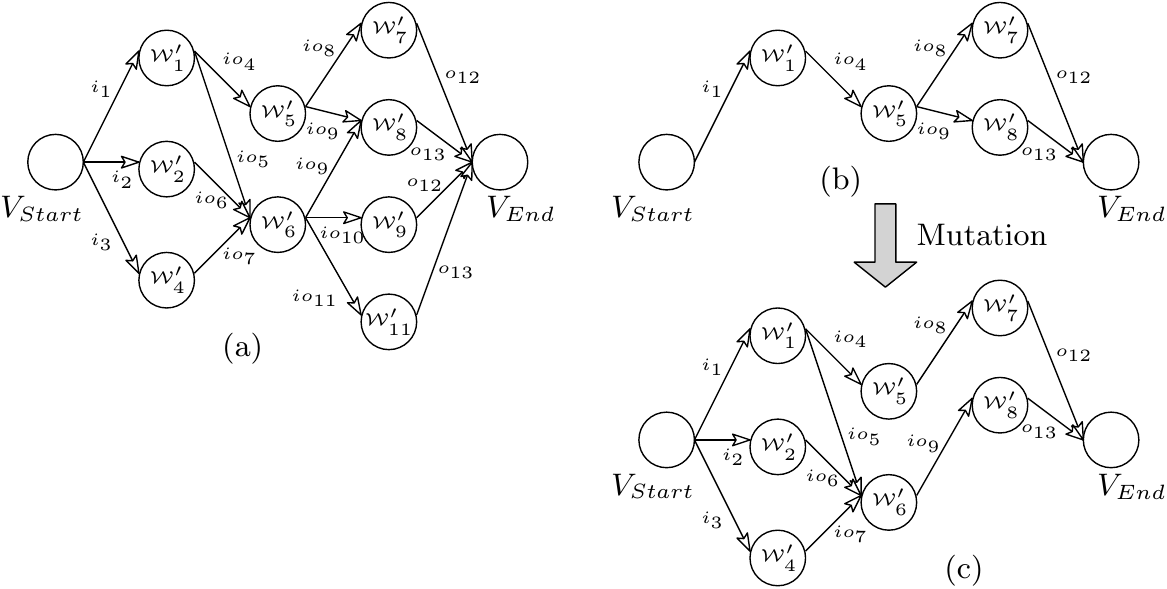}
  \caption{Example of mutation}
  \label{fig:mutation}
 \end{figure}
 \hfill$\blacksquare$
\end{example}


\subsection{Fitness function}
\noindent
In this section, we discuss the computation of the fitness function for each chromosome in the population. For each chromosome in the population, we first compute the QoS tuple. The chromosomes are then divided into different levels based on domination. A chromosome $CR_i$ belongs to level $l_j$, if there exists at-least one chromosome in each level $l_k$, where $k < j$, that dominates $CR_i$ and no chromosome belonging to level $l_m$, where $m \ge j$ dominates $CR_i$. For chromosomes belonging to the same level, we compute the crowding distance $Distance (CR_i)$ as illustrated in \cite{deb2002fast}. Crowding distance ensures solution diversity. This essentially measures how close the individual chromosome is from its neighbors with respect to each QoS parameter. 
Finally, the chromosomes are ranked as follows:

Between two chromosomes $CR_i$ and $CR_j$, rank of $CR_i$ is higher than the rank $CR_j$ if $(Level (CR_i) < Level (CR_j))$ or ($(Level (CR_i) == Level (CR_j))$ and $(Distance (CR_i) > Distance (CR_j))$). The rank of a chromosome is treated as the fitness value of the chromosome. Less fitness value represents better solution quality. 

When algorithm terminates, all the phenotypes corresponding to the chromosomes with level 0 are returned as the solutions. In the next section, we experimentally shows the comparative study of our approaches.

%% file: result.tex
\section{Experimental Results}\label{sec:result}
\noindent
We implemented our proposed framework in Java. Experiments were performed on an i3 processor with 4GB RAM. The algorithms were evaluated on the 19 public repositories of the ICEBE-2005 Web Service Challenge \cite{icebe2005} and the 5 public repositories of the 2009-2010 Web Service Challenge (WSC 2009-10) \cite{bansal2009wsc} and the dataset demonstrated in \cite{7933195}.
The WSC 2009-10 dataset contains only the values of two QoS parameters (response time and throughput) for each service. Additionally, we randomly generated the values for reliability and availability for each service. However, the ICEBE-2005 dataset does not contain the value of any QoS parameter for the services. We randomly generated values of all 4 QoS parameters (response time, throughput, reliability and availability) for our experiments with the ICEBE-2005 dataset.

\noindent
{\textbf{Configurations of our algorithms}}: No configuration parameter is required for the optimal algorithm. For our first heuristic approach, we varied the beam width from 100 to 500. Finally, we compare our first heuristic with \cite{yan2012anytime} considering the beam width as 500. For our second heuristic algorithm, we chose the following configuration parameters: 
population size as 100, the mutation probability as 0.01 and crossover probability as 0.85 and the number of iterations as 10000.

\input{resultSub}

\subsection{Different metrics to compare results}
\noindent
We use the following metrics to compare our results with other approaches in literature, as defined below.
\begin{itemize}
 \item $n$: The cardinality of the solutions obtained from an algorithm.
 \item Commonality Ratio ${\cal{CR}}({\cal{\hat{T}}}_1, {\cal{\hat{T}}}_2)$: Given two sets of solution tuples ${\cal{\hat{T}}}_1$ and ${\cal{\hat{T}}}_2$, 
 {\scriptsize{\[
 {\cal{CR}}({\cal{\hat{T}}}_1, {\cal{\hat{T}}}_2) = \frac{|{\cal{\hat{T}}}_1 \cap {\cal{\hat{T}}}_2|}
 {|{\cal{\hat{T}}}_1 \cup {\cal{\hat{T}}}_2|}
 \]}}
 \item Commonality Non Dominated Solution Ratio ${\cal{CN}}({\cal{\hat{T}}}_1, {\cal{\hat{T}}}_2)$: Given two sets of solution tuples ${\cal{\hat{T}}}_1$ and ${\cal{\hat{T}}}_2$,  
 {\scriptsize{\[
 {\cal{CN}}({\cal{\hat{T}}}_i, {\cal{\hat{T}}}_3) = \frac{|{\cal{\hat{T}}}_i \cap {\cal{\hat{T}}}_3|}{|{\cal{\hat{T}}}_3|}; i=1,2;
 \]}}
 ${\cal{\hat{T}}}_3 = $ set of non dominated tuples obtained from ${\cal{\hat{T}}}_1 \cup {\cal{\hat{T}}}_2$; 
 \item Average Distance Ratio ${\cal{AD}}({\cal{\hat{T}}}_1, {\cal{\hat{T}}}_2)$: Given two sets of solution tuples ${\cal{\hat{T}}}_1$ and ${\cal{\hat{T}}}_2$
 {\scriptsize{\[
 {\cal{AD}}({\cal{\hat{T}}}_1, {\cal{\hat{T}}}_2) = \frac{\frac{1}{|{\cal{\hat{T}}}_1|}\sum_{t_i \in {\cal{\hat{T}}}_1}U(t_i)}{\frac{1}{|{\cal{\hat{T}}}_2|}\sum_{t_i \in {\cal{\hat{T}}}_2}U(t_i)}
 \]}}
 $U(t_i)$ is calculated as in Eq. \ref{equ:utility} for all tuples in ${\cal{\hat{T}}}_1 \cup {\cal{\hat{T}}}_2$.
 \item Speed up $S(A_1, A_2)$: Given two algorithms $A_1$ and $A_2$,
 {\scriptsize{\[
 S(A_1, A_2) = \frac{\text{Computation time for } A_2}{\text{Computation time for } A_1}
 \]}}
\end{itemize}

\noindent
These metrics are used to compare the solutions obtained by two different algorithms or the same algorithm with different configurations. The {\em{average distance ratio}} metric shows the quality difference between two solutions. It may be noted, ${\cal{AD}}({\cal{\hat{T}}}_1, {\cal{\hat{T}}}_2) > 1$ means ${\cal{\hat{T}}}_1$ provides better result than ${\cal{\hat{T}}}_2$, while ${\cal{AD}}({\cal{\hat{T}}}_1, {\cal{\hat{T}}}_2) < 1$ means ${\cal{\hat{T}}}_2$ provides better result than ${\cal{\hat{T}}}_1$. The {\em{commonality ratio}} indicates how many tuples are common in two solutions, where as the {\em{commonality non dominated ratio}} shows how many non-dominated tuples are common.

\subsection{Analysis on a synthetic dataset}
\noindent
We first analyze the performance of our proposed algorithms on a synthetic dataset demonstrated in \cite{7933195}. Here we compare the performance of the heuristic methods with respect to the optimal one. The total number of services in the service repository was 567. We used the QWS \cite{Al-Masri:2007:DBW:1242572.1242795} dataset to assign the QoS values to the services. The QWS dataset has 8 different QoS parameters and more than 2500 services. From the QWS dataset, we randomly selected 567 services and the corresponding QoS values were assigned to the services in our repository. After preprocessing, we had only 164 different services. The Pareto optimal algorithm with preprocessing achieved $6.36$ times speed up in comparison to the Pareto optimal algorithm without any preprocessing.

The number of solutions obtained by the Pareto optimal algorithm ($PO$), our first heuristic ($H_1$) and our second heuristic ($H_2$) are 7, 8, 7 respectively. While $H_1$ generated 6 non dominated tuples, $H_2$ generated 5 non dominated tuples common with the tuples generated by $PO$. Average distance ratio between $PO$ and $H_1$ is 1.38, whereas, the average distance ratio between $PO$ and $H_2$ is 1.63. Our first heuristic method with beam width 500 achieved $31$ times speed up in comparison to the Pareto optimal algorithm with preprocessing, whereas, our second heuristic method achieved $43$ times speed up in comparison to the Pareto optimal algorithm with preprocessing. It may be noted, though our first heuristic algorithm generated better quality result with respect to our second heuristic algorithm, however, the second one provided higher speed up than our first one.

\subsection{Analysis on public datasets}
\noindent
We now show the experimental results obtained by our algorithm on ICEBE-2005 and WSC 2009-10.

\subsubsection{Results of Preprocessing}
\noindent
Given a service repository as an input to our problem, we first present the details of the reduction obtained in the number of services after the preprocessing step in our methodology on the benchmark datasets. Table \ref{tab:reduction} presents the summary on the ICEBE-2005 dataset. However, it is interesting to note that no reduction could be obtained for the WSC 2009-10 dataset, {\textcolor{black}{since the dataset does not contain any equivalent service.}}

\begin{table}[!ht]
\scriptsize
\caption{Reduction after preprocessing for ICEBE-2005}
\centering
\begin{tabular}{l|c|c|l|c|c}
  Data Set & $N_1$ &  $N_2$ & Data Set & $N_1$ & $N_2$ \\
 \hline
 Out Composition & 143 & 79  &
Composition1-20-4 & 2156 & 2032 \\ 
Composition1-20-16 & 2156 & 2143 &
Composition1-20-32 & 2156 & 2152 \\ 
Composition1-50-4 & 2656 & 2596  &
Composition1-50-16 & 2656 & 2652 \\ 
Composition1-50-32 & 2656 & 2656  &
Composition1-100-4 & 4156 & 4081 \\ 
Composition1-100-16 & 4156 & 4150 & 
Composition1-100-32 & 4156 & 4155 \\ 
Composition2-20-4 & 3356 & 3195  &
Composition2-20-16 & 6712 & 6678 \\ 
Composition2-20-32 & 3356 & 3347  &
Composition2-50-4 & 5356 & 5239 \\ 
Composition2-50-16 & 5356 & 5346 & 
Composition2-50-32 & 5356 & 5349 \\ 
Composition2-100-4 & 8356 & 8233  &
Composition2-100-16 & 8356 & 8347 \\ 
Composition2-100-32 & 8356 & 8354 & & &\\ \hline
\multicolumn{6}{c}{$N_1$ and $N_2$ represent the number of services before and after preprocessing}\\
\end{tabular}\label{tab:reduction}
\end{table}

\subsubsection{Runtime performance analysis}
\noindent
We now analyze the performance of our heuristic algorithms and show the trade-off between computation time and solution quality. Tables \ref{tab:compareAnyTimeWSC} and \ref{tab:compareAnyTimeICEBE} show the comparative results obtained by the first heuristic algorithm ($H_1$) for WSC-2009 and ICEBE-2005 datasets respectively. We gradually increased the size of the beam width of the algorithm and generated the solution. We consider three different cases depending on the size of the beam width: 100, 300 and 500 respectively. The last column of Table \ref{tab:compareAnyTimeWSC} and \ref{tab:compareAnyTimeICEBE} present the computation time required by $H_1$ to generate the solution. We compare the results obtained by the algorithm in Case $i$ with the same in Case $(i + 1)$, for $i = 1, 2$. For example, the results obtained by $H_1$ with beam width 100 is compared with the same with beam width 300. As we have already discussed in Section \ref{sec:method}, as the beam size increases, either the solution quality remains same or improves. In this comparative study, we show the degradation of the solution quality when the beam size decreases in terms of ${\cal{AD}}$, ${\cal{CR}}$ and ${\cal{CN}}$.


Finally, we compare our heuristic algorithms with \cite{yan2012anytime} which transforms the multiple objectives to a single objective to generate a single solution. In contrast, our method is able to generate the multiple non-dominated feasible solutions in a comparable time limit. Table \ref{tab:compareOurAnyTimeICEBE} 
shows the speed up achieved by both the methods with respect to \cite{yan2012anytime}. As is evident from Columns 2, 3, 5 and 6 of Table \ref{tab:compareOurAnyTimeICEBE} 
that in some cases, we have achieved speed-up more than 1 in 15 and 14 cases for $H_1$ and $H_2$ respectively out of 23 cases, which implies our method performs better than the method in \cite{yan2012anytime} in terms of computation time.

%

The Pareto optimal algorithm, being compute intensive, does not produce any result and encounters a memory out error on both the datasets, on our machine with 4GB RAM. We now show a comparative performance analysis for both the heuristic methods. 
Table \ref{tab:compareHeuristicAnyTimeICEBE} shows the comparison between both the methods for the ICEBE-2005 and WSC-2009 datasets respectively. As evident from the tables, $H_1$ and $H_2$ generated the same results in 8 cases, $H_1$ generated better results than $H_2$ in 7 cases, while $H_2$ generated better results than $H_1$ in 9 cases. In 15 cases $H_1$ achieved more speed up than $H_2$, while in 9 cases $H_2$ outperformed $H_1$ in terms of speed up. As evident from our result, both the heuristic performed well and achieved more speed-up.

\begin{table}[!ht]
\scriptsize
\caption{{\scriptsize{Comparison between our methods and \cite{yan2012anytime} for ICEBE-2005 (Rows 2-10) and WSC 2009 (Rows 11-13)}}}
\centering
\begin{tabular}{c|c|c|c}
Dataset & $S_1, S_2$ & Dataset & $S_1, S_2$\\ 
\hline
Composition1-20-4   & 0.90, 0.18 & Composition2-20-4   & 1.75, 2.54\\ 
Composition1-20-16  & 1.09, 0.13 & Composition2-20-16  & 1.48, 1.72\\ 
Composition1-20-32  & 0.78, 0.2 & Composition2-20-32  & 0.68, 0.85\\
Composition1-50-4   & 1.51, 0.52 & Composition2-50-4   & 0.70, 0.38\\
Composition1-50-16  & 0.68, 0.4 & Composition2-50-16  & 0.85, 2.83\\
Composition1-50-32  & 1.03, 0.27 & Composition2-50-32  & 7.85, 11.06\\
Composition1-100-4  & 0.57, 2.48 & Composition2-100-4  & 1.17, 0.6\\
Composition1-100-16 & 8.85, 0.9 & Composition2-100-16 & 3.8, 7.04\\ 
Composition1-100-32 & 3.55, 1.8 & Composition2-100-32 & 7.11, 13.67\\ 
\hline
WSC-01 & 0.84, 1.3 & WSC-04 & 1.68, 0.45 \\ 
WSC-02 & 7.09, 3.97 & WSC-05 & 5.57, 1.67\\
WSC-03 & 11.11, 3.88 &&\\\hline
\multicolumn{4}{l}{$S_1 = S(H_1,$\cite{yan2012anytime} $)$, $S_2 = S(H_2,$ \cite{yan2012anytime}$)$}
\end{tabular}\label{tab:compareOurAnyTimeICEBE}
\end{table}

\begin{table}[!ht]
\scriptsize
\caption{\scriptsize{Comparison between our first (with beam width 500) and second heuristic methods for ICEBE-2005 (Rows 2-20) and WSC-2009 (Rows 21-25)}}
\centering
\begin{tabular}{c|c|c|c|c|c}
Dataset             & n & ${\cal{AD}}$ & ${\cal{CR}}$ & ${\cal{CN}}$ & Speed-up\\ 
\hline
Out Composition     & (3, 4)  & 0.96     & 1    & (0.23, 0.39) & 1.86\\ 
Composition1-20-4   & (7, 2)  & 2.56     & 0.59 & (0.81, 0.69) & 4.93\\ 
Composition1-20-16  & (5, 9)  & 1.71     & 0.63 & (0.9, 0.58)  & 0.69\\ 
Composition1-20-32  & (4, 4)  & 1        & 1    & (1,1)        & 8.61\\ 
Composition1-50-4   & (9, 7)  & 0.82     & 0.67 & (0.91, 0.96) & 0.86\\ 
Composition1-50-16  & (6, 5)  & 1.92     & 0.33 & (0.68, 0.57) & 3.91\\ 
Composition1-50-32  & (2, 2)  & 1        & 1    & (1,1)        & 0.8\\ 
Composition1-100-4  & (11, 11)& 1        & 1    & (1,1)        & 2.93\\ 
Composition1-100-16 & (8, 6)  & 0.9      & 0.67 & (0.33, 0.33) & 1.85\\ 
Composition1-100-32 & (3, 3)  & 1        & 1    & (1,1)        & 1.68\\ 
Composition2-20-4   & (7, 9)  & 0.69     & 0.42 & (0.72, 0.96) & 0.3\\ 
Composition2-20-16  & (9, 6)  & 0.91     & 0.85 & (0.63, 0.76) & 3.86\\ 
Composition2-20-32  & (11,9)  & 0.68     & 0.98 & (0.61, 0.8)  & 0.71\\ 
Composition2-50-4   & (2, 5)  & 3.93     & 0.67 & (0.85, 0.35) & 0.23\\ 
Composition2-50-16  & (8, 11) & 2.91     & 0.35 & (0.96, 0.62) & 1.96\\ 
Composition2-50-32  & (3, 7)  & 0.86     & 0.85 & (0.65, 0.5)  & 9.8\\ 
Composition2-100-4  & (6, 6)  & 1        & 1    & (1,1)        & 0.54\\ 
Composition2-100-16 & (8, 6)  & 0.8      & 0.69 & (0.85, 0.92) & 1.97\\ 
Composition2-100-32 & (5, 4)  & 1.93     & 0.71 & (0.96, 0.81) & 0.52\\ 
\hline
WSC-01 & (2, 2) & 1     & 1    & (1, 1) & 0.63\\ 
WSC-02 & (2, 2) & 1     & 1    & (1, 1) & 1.79\\ 
WSC-03 & (2, 2) & 1     & 1    & (1, 1) & 2.86\\ 
WSC-04 & (5, 6) & 0.92  & 0.83 & (0.83, 1) & 3.71\\ 
WSC-05 & (6, 4) & 1.87  & 0.27 & (0.8, 0.6) & 3.33\\\hline
\multicolumn{6}{l}{~}\\
\multicolumn{6}{l}{${\cal{CR}}(Tuple_{H_1}, Tuple_{H_2})$, (${\cal{CN}}(Tuple_{H_1}, {\cal{\hat{T}}}_3), {\cal{CN}}(Tuple_{H_2},{\cal{\hat{T}}}_3))$}\\
\multicolumn{6}{l}{${\cal{AD}}(Tuple_{H_1}, Tuple_{H_2})$, $n$ is written as $(|Tuple_{H_1}|, |Tuple_{H_2}|)$, $S({H_1}, {H_2})$}
\end{tabular}\label{tab:compareHeuristicAnyTimeICEBE}
\end{table}

%

%% file: resultSub.tex
\begin{table*}[!ht]
\scriptsize
\caption{\scriptsize{Comparison between solutions with increasing beam width for WSC-2009 (Case $i$ computed with respect to Case $i + 1$ )}}
\centering
\begin{tabular}{c||c|c|c|c|c||c|c|c|c|c||c|c}
\hline
&\multicolumn{12}{c}{Beam Width}\\
\cline{2-13}
Data Set  & \multicolumn{5}{c||}{Case 1: 100} & \multicolumn{5}{c||}{Case 2: 300} & \multicolumn{2}{c}{Case 3: 500} \\
\cline{2-13}
& n & ${\cal{AD}}$ & $\cal{CR}$ & $\cal{CN}$ & S & n & ${\cal{AD}}$ & $\cal{CR}$ & $\cal{CN}$ & S & n & Computation Time (sec) \\
 \hline
WSC-01 & 1    & 0.86 & 0.5 & 0.5 & 5.83 & 2 & 1 & 1 & 1 & 7.86 & 2 & 8\\ \hline
WSC-02 & 2    & 1 & 1 & 1 & 8.69 & 2 & 1 & 1 & 1 & 9.32 & 2 & 11 \\ \hline
WSC-03 & 2    & 0.79 & 0.33 & 0.5 & 21.71 & 2 & 1 & 1 & 1 & 7.86 & 2 & 126 \\ \hline
WSC-04 & 2    & 0.93 & 0.4 & 0.4 & 29.74 & 3 & 0.81 & 0.33 & 0.4 & 18.86 & 5 & 397 \\ \hline
WSC-05 & 3    & 0.89 & 0.5 & 0.5 & 37.21 & 5 & 0.98 & 0.83 & 0.83 & 11.86 & 6 & 778\\
\end{tabular}\label{tab:compareAnyTimeWSC}
\end{table*}

\begin{table*}[!ht]
\scriptsize
\caption{\scriptsize{Comparison between solutions with increasing beam width for ICEBE-2005 (Case $i$ computed with respect to Case $i + 1$ )}}
\centering
\begin{tabular}{c||c|c|c|c|c||c|c|c|c|c||c|c}
\hline
&\multicolumn{12}{c}{Beam Width}\\
\cline{2-13}
Data Set  & \multicolumn{5}{c||}{Case 1: 100} & \multicolumn{5}{c||}{Case 2: 300} & \multicolumn{2}{c}{Case 3: 500} \\
\cline{2-13}
& n & ${\cal{AD}}$ & $\cal{CR}$ & $\cal{CN}$ & S & n & ${\cal{AD}}$ & $\cal{CR}$ & $\cal{CN}$ & S & n & Computation Time (sec) \\
 \hline
Out Composition     & 3    & 1 & 1 & 1 & 3.89   & 3 & 1    & 1 & 1 & 4.39    & 3 & 0.51 \\ \hline
Composition1-20-4   & 1    & 0.92 & 0.25 & 0.25 & 1.63   & 4 & 0.88 & 0.57 & 0.57 & 7.36    & 7 & 2 \\ \hline
Composition1-20-16  & 2    & 0.97 & 0.4 & 0.4 & 9.85   & 5 & 0.96 & 0.67 & 0.8 & 6.13    & 5 & 7 \\ \hline
Composition1-20-32  & 4    & 0.91 & 0.29 & 0.4 & 13.12  & 5 & 0.99 & 0.5 & 0.75 & 13.15   & 4 & 11 \\ \hline
Composition1-50-4   & 5    & 0.84 & 0.62 & 0.62 & 17.81  & 8 & 0.76 & 0.89 & 0.89 & 21.15   & 9 & 10 \\ \hline
Composition1-50-16  & 5    & 0.8 & 0.4 & 0.5 & 5.46   & 4 & 0.82 & 0.25 & 0.33 & 9.76    & 6 & 18 \\ \hline
Composition1-50-32  & 1    & 0.75 & 0.5 & 0.5 & 9.81   & 2 & 1    & 1 & 1 & 18.71   & 2 & 17 \\ \hline
Composition1-100-4  & 6    & 0.94 & 0.56 & 0.62 & 6.75   & 8 & 0.9  & 0.73 & 0.73 & 3.42    & 11 & 38 \\ \hline
Composition1-100-16 & 2    & 0.9 & 0.33 & 0.33 & 2.96   & 6 & 0.84 & 0.4 & 0.5 & 11.89   & 8 & 63 \\ \hline
Composition1-100-32 & 2    & 1 & 1 & 1 & 8.49   & 2 & 0.89 & 0.67 & 0.67 & 5.39    & 3 & 69 \\ \hline
Composition2-20-4   & 5    & 0.84 & 0.43 & 0.33 & 13.17  & 7 & 0.87 & 0.56 & 0.71 & 11.16   & 7 & 51 \\ \hline
Composition2-20-16  & 3    & 0.83 & 0.6 & 0.6 & 6.19   & 5 & 0.69 & 0.17 & 0.22 & 4.96    & 9 & 308\\ \hline
Composition2-20-32  & 7    & 0.9 & 0.54 & 0.54 & 19.1   & 13& 0.93 & 0.33 & 0.54 & 13.59   & 11 & 97 \\ \hline
Composition2-50-4   & 9    & 0.69 & 0.27 & 0.6 & 11.31  & 5 & 0.79 & 0 & 0 & 8.97    & 2 & 241 \\ \hline
Composition2-50-16  & 2    & 0.99 & 0.67 & 0.67 & 4.87   & 3 & 0.91 & 0.37 & 0.37 & 9.36    & 8 & 306 \\ \hline
Composition2-50-32  & 3    & 1 & 1 & 1 & 2.91   & 3 & 1    & 1 & 1 & 12.82   & 3 & 531 \\ \hline
Composition2-100-4  & 4    & 0.96 & 0.8 & 0.8 & 8.95   & 5 & 0.91 & 0.83 & 0.83 & 27.26   & 6 & 649 \\ \hline
Composition2-100-16 & 5    & 0.83 & 0.57 & 0.5 & 21.37  & 7 & 0.8  & 0.62 & 0.5 & 9.86    & 8 & 1339 \\ \hline
Composition2-100-32 & 2    & 0.91 & 0.33 & 0.33 & 36.59  & 6 & 0.77 & 0.22 & 0.4 & 3.59    & 5 & 1689 \\
\end{tabular}\label{tab:compareAnyTimeICEBE}
\end{table*}

%% file: conclusion.tex
\section{Conclusion and future directions}
\label{sec:conclusion}
\noindent
This paper addresses the problem of multi-constrained multi-objective service composition in IOM based on the Pareto front construction. 
Experimental results on real benchmarks show the effectiveness of our proposal. We believe that our work will open up a lot of new research directions in the general paradigm of multi-objective service composition. Going forward, we wish to come up with a theoretical bound on the solution quality degradation of our heuristic algorithms. 

%% file: appendix.tex
\begin{appendix} 
\begin{lemma}
 The preprocessing step is Pareto optimal solution preserving in terms of QoS values.
 \hfill$\blacksquare$
\end{lemma}

 \begin{proof}
  
  We prove the above lemma by contradiction. 
  Consider a solution $S_i$ to a query ${\cal{Q}}$ is lost due to preprocessing.
  We show either of the following two conditions must hold:
  
  \begin{itemize}
   \item A solution $S_j$ to the query ${\cal{Q}}$, generated after preprocessing, dominates $S_i$.
   \item A solution $S_j$ to the query ${\cal{Q}}$, generated after preprocessing, has exactly the same QoS tuple as $S_i$.
  \end{itemize}
 
  \noindent
  We first assume that none of the above conditions hold for $S_i$.  
  Since $S_i$ is lost due to preprocessing, there must exist a solution consisting of at least one service 
  ${\cal{W}}_k$, which is removed during preprocessing. This further implies, presence of another service 
  ${\cal{W}}_l$, belonging to the same cluster as ${\cal{W}}_k$, such that ${\cal{W}}_l$ dominates ${\cal{W}}_k$. 
  We now analyze different cases for different aggregation functions for different QoS parameters. Consider 
  a QoS parameter ${\cal{P}}_x$ with aggregation function $f_x$. 
  
   If ${\cal{P}}_x$ is a positive QoS, ${\cal{P}}^{(l)}_x \ge {\cal{P}}^{(k)}_x$.
 	If ${\cal{P}}^{(l)}_x = {\cal{P}}^{(k)}_x$, $f_x(P^*, {\cal{P}}^{(l)}_x) = f_x(P^*, {\cal{P}}^{(k)}_x)$, 
 	where $P^*$ denotes the values of ${\cal{P}}_k$ for the set of services with which 
 	${\cal{P}}^{(k)}_x$ is combined in $S_i$. We, therefore, consider the cases where 
 	${\cal{P}}^{(l)}_x > {\cal{P}}^{(k)}_x$.
   \begin{itemize}
    \item $f_x$ is addition: $Sum(P^*, {\cal{P}}^{(l)}_x) > Sum(P^*, {\cal{P}}^{(k)}_x)$.
    \item $f_x$ is product: $Prod(P^*, {\cal{P}}^{(l)}_x) > Prod(P^*, {\cal{P}}^{(k)}_x)$.
    \item $f_x$ is maximum: 
 	\begin{itemize}
 	 \item If $Max(P^*, {\cal{P}}^{(k)}_x) = {\cal{P}}^{(k)}_x$, \\
 		  $Max(P^*, {\cal{P}}^{(l)}_x) > Max(P^*, {\cal{P}}^{(k)}_x)$.
 	 \item If ${\cal{P}}^{(k)}_x < Max(P^*, {\cal{P}}^{(k)}_x) < {\cal{P}}^{(l)}_x$, \\
 		  $Max(P^*, {\cal{P}}^{(l)}_x) > Max(P^*, {\cal{P}}^{(k)}_x)$.
 	 \item If ${\cal{P}}^{(k)}_x < Max(P^*, {\cal{P}}^{(k)}_x)$ and 
 		  ${\cal{P}}^{(l)}_x < Max(P^*, {\cal{P}}^{(k)}_x)$, \\
 		  $Max(P^*, {\cal{P}}^{(l)}_x) = Max(P^*, {\cal{P}}^{(k)}_x)$.
 	\end{itemize}
 	Therefore, $Max(P^*, {\cal{P}}^{(l)}_x) \ge Max(P^*, {\cal{P}}^{(k)}_x)$.
    \item $f_x$ is minimum:  
 	\begin{itemize}
 	 \item If $Min(P^*, {\cal{P}}^{(k)}_x) = {\cal{P}}^{(k)}_x$ and $Min(P^*) > {\cal{P}}^{(k)}_x$,\\
 		  $Min(P^*, {\cal{P}}^{(l)}_x) > Min(P^*, {\cal{P}}^{(k)}_x)$.
 	 \item If $Min(P^*) \le {\cal{P}}^{(k)}_x$,\\
 		  $Min(P^*, {\cal{P}}^{(l)}_x) = Min(P^*, {\cal{P}}^{(k)}_x)$.
 	\end{itemize}
 	Therefore, $Min(P^*, {\cal{P}}^{(l)}_x) \ge Min(P^*, {\cal{P}}^{(k)}_x)$.
   \end{itemize}
  Using a similar argument, it can be shown that the above lemma also holds for any
  negative QoS parameter. Therefore, the solution consisting of ${\cal{W}}_l$ is {\em{at least as good as}} the solution 
  consisting of ${\cal{W}}_k$, which contradicts our assumption.
 \end{proof}

\begin{lemma}
 Each path from $V^{(p)}_{Start}$ to $V^{(p)}_{End}$ in $G_P$ represents a solution to the query in terms 
 of functional dependencies.
 \hfill$\blacksquare$
\end{lemma}

\begin{proof}
  The proof of the lemma follows from the construction of $G_P$.
  Consider a path $\rho: V^{(p)}_{Start}-v^{(p)}_{i_1}-v^{(p)}_{i_2}-\ldots-v^{(p)}_{i_k}-V^{(p)}_{End}$ 
  of $G_P$. The services corresponding to $v^{(p)}_{i_1}$ are directly activated by the query inputs, 
  while the services corresponding to $v^{(p)}_{i_2}$ are activated by outputs of the services 
  corresponding to $v^{(p)}_{i_1}$. Similarly, the services corresponding to $v^{(p)}_{i_j}$ for 
  $j= 2, 3 \ldots, k$ are activated by the outputs of the services corresponding to $v^{(p)}_{i_{(j - 1)}}$. 
  Finally, the services corresponding to $v^{(p)}_{i_k}$ produce the query outputs. Therefore, the services 
  associated with the path $\rho$ forms a composition solution in terms of functional dependency. 
\end{proof}

\begin{lemma}
 Algorithm 4 is complete.
 \hfill$\blacksquare$
\end{lemma}

 \begin{proof}
  In order to prove the completeness of Algorithm 4, we show Algorithm 
  4 generates all distinct feasible solutions in terms of QoS values belonging to 
  the Pareto front in response to a query ${\cal{Q}}$. We prove the above lemma by contradiction. We first assume 
  that a unique (in terms of QoS values) feasible solution $S_i$ belonging to the Pareto front is not generated 
  by Algorithm 4. This implies, one of the following condition holds:
  \begin{itemize}
   \item Case 1: $S_i$ cannot be generated through any path of $G_P$, even if no tuple or node is removed during 
 	solution construction.
   \item Case 2: $S_i$ is removed while constructing the set of non dominated tuples corresponding to a node in $G_P$.
   \item Case 3: $S_i$ is removed while constructing the cumulative Pareto optimal solution till a node $v^{(p)}_i \in G_P$.
  \end{itemize}
  Since we have already proved in Lemma \ref{lemma:pareto} that the preprocessing step is Pareto optimal solution 
  preserving in terms of QoS values, we can now assume that $S_i$ consists of the services
  in $W'$. We now analyze each case separately below.
  
   Consider Case 1. Below, we prove that Case 1 cannot be true by the following argument. 
 	    While constructing the set of predecessor nodes of a node $v^{(p)}_i \in G_P$, we consider all 
 	    possible combinations of nodes in ${\cal{D}}$ which can activate the nodes
 	    corresponding to $v^{(p)}_i$. 	    
 	    Consider the subgraph ${\cal{D}}_{Sub}$ of ${\cal{D}}$ corresponding to the solution $S_i$ and the
 	    partition ${\cal{D}}_{Sub}$ into multiple layers as defined by Equations 1
 	    and 2. The last layer (say $L_k$) of ${\cal{D}}_{Sub}$ consists of $V_{End} \in {\cal{D}}$. 
 	    Consider each layer $L_i$, for $i = 0, 1, \ldots, k$ in ${\cal{D}}_{Sub}$ consists of a set of nodes 
 	    $V_i \subset V$. It may be noted, the set of services corresponding to $V_{(k-1)}$ 
 	    in $L_{(k - 1)}$ produce the query outputs. $V^{(p)}_{End} \in G_P$ consists of $V_{End} \in {\cal{D}}$. 
 	    While constructing the set of predecessor nodes of $V^{(p)}_{End}$, since all possible combinations of nodes 
 	    in ${\cal{D}}$ which can activate $V_{End}$ are considered, $V_{(k-1)}$ is also considered. 
 	    Therefore, a node $v^{(p)}_{i_1} \in G_P$ must correspond to $V_{(k-1)}$. 
 	    In general, $v^{(p)}_{i_j} \in G_P$ corresponds to $V_{(k-j)}$. 
 	    Since all possible combinations of nodes in ${\cal{D}}$ which can activate $V_{(k-j)}$
 	    are considered while constructing the set of predecessor nodes of $V_{(k-j)}$,
 	    $V_{(k - j - 1)}$ is also considered. Hence, a path corresponding to ${\cal{D}}_{Sub}$ belongs to $G_P$, which 
 	    contradicts our assumption.
 	    
   We now consider Case 2. Consider $v^{(p)}_i \in G_P$ consists of $\{v_{i_1}, v_{i_2}, \ldots, v_{i_k}\} \in V$ of ${\cal{D}}$.
 	    The services corresponding to $\{v_{i_1}, v_{i_2}, \ldots, v_{i_k}\}$ are executed in parallel. It may be 
 	    noted, each service corresponding to $v_{i_j} \in \{v_{i_1}, v_{i_2}, \ldots, v_{i_k}\}$ may be associated with 
 	    more than one QoS tuple. Therefore, after composition, more than one tuple is generated. We consider only 
 	    the set of non dominated tuples {\small{${\cal{TP}}^*$}} from the set of generated tuples {\small{${\cal{TP}}$}}. If 
 	    $S_i$ is removed due to removal of {\small{$({\cal{TP}} \setminus {\cal{TP}}^*)$}}, this implies $S_i$ 
 	    consists of one QoS tuple {\small{$t \in ({\cal{TP}} \setminus {\cal{TP}}^*)$}}. This further implies, 
 	    another tuple $t' \in$ {\small{${\cal{TP}}^*$}} must dominate $t$. By a similar argument, as in the proof of 
 	    Lemma \ref{lemma:pareto}, we can say another solution $S_j$ generated by Algorithm 4
 	    either has exactly the same QoS tuple as $S_i$, which implies $S_i$ is not unique, or $S_j$ dominates $S_i$, which
 	    contradicts our assumption. 	    
 	    The argument of Case 3 is same as in Case 2.
  Therefore, none of the above cases hold, which contradicts our assumption. Thus Algorithm 4 is 
  complete.
\end{proof}

\begin{lemma}
 Algorithm 4 is sound.
 \hfill$\blacksquare$
\end{lemma}

 \begin{proof}
  In order to prove the soundness of Algorithm 4, we need to show each solution $S_i$,
  generated by Algorithm 4 in response to a query ${\cal{Q}}$, is a feasible solution 
  and belongs to the Pareto front.  
  If $S_i$ is generated by Algorithm 4, it must be a feasible solution, 
  as it is ensured by Steps 18-20 
  of Algorithm 4. We now prove by contradiction that $S_i$ belongs to the Pareto front.
  We first assume $S_i$ does not belong to the Pareto front. This implies, there exists another solution $S_j$ 
  which dominates $S_i$. However, being complete, Algorithm 4 generates all the feasible
  solutions belonging to the Pareto front. Therefore, Algorithm 4 generates $S_j$ as  
  one of the solutions. Therefore, $S_i$ cannot be generated by Algorithm 4, since $S_j$
  dominates $S_i$, which contradicts our assumption. Thus, Algorithm 4 is sound.
 \end{proof}
 
 \begin{lemma}
The solution quality of the heuristic algorithm monotonically improves with increase in beam width.
 \hfill$\blacksquare$
\end{lemma}

\begin{proof}
 The heuristic algorithm uses Algorithm 4 to generate the solutions from the reduced search space. Therefore, the solutions generated by this algorithm is the feasible Pareto optimal solutions obtained from the reduced search space according to Lemma \ref{lemma:complete} and \ref{lemma:sound}. Moreover the search space required by the heuristic algorithm with beam width equals to $k$ subsumes the search space required by the heuristic algorithm with beam width equals to $k_1$, where $k > k_1$. Therefore, with the increase in beam width, the solution quality of the heuristic algorithm monotonically improves.
\end{proof}
 
\end{appendix}